\documentclass[letterpaper]{article}

\PassOptionsToPackage{numbers, compress}{natbib}  
\usepackage[final]{neurips_2022}

\usepackage{enumitem}
\usepackage[utf8]{inputenc} 
\usepackage[T1]{fontenc}    
\usepackage{hyperref}       
\usepackage{url}            
\usepackage{booktabs}       
\usepackage{amsfonts}       
\usepackage{nicefrac}       
\usepackage{microtype}      
\usepackage{xcolor}         

\usepackage{textcomp} 
\usepackage[english]{babel} 
\usepackage[autolanguage]{numprint} 
\usepackage{hyperref} 
\usepackage{graphicx} 
\usepackage{verbatim} 
\usepackage{amsmath,amssymb,amsthm} 
\usepackage{dsfont}
\usepackage[linesnumbered,ruled]{algorithm2e} 
\usepackage{xspace}
\usepackage{tikz}

\theoremstyle{definition} 
\newtheorem{Def}{Definition}[section] 
\theoremstyle{plain} 
\newtheorem{Lem}[Def]{Lemma} 
\newtheorem{The}[Def]{Theorem} 
\newtheorem{Cor}[Def]{Corollary} 
\theoremstyle{remark} 

\newcommand{\N}{\mathbb{N}}
\newcommand{\R}{\mathbb{R}}
\newcommand{\E}{{\mathbb{E}}}
\newcommand{\Pb}{{\mathbb{P}}}
\newcommand{\1}{{\mathds{1}}}
\newcommand{\prt}[1]{\left(#1\right)}
\newcommand{\brac}[1]{\left[#1\right]}
\newcommand{\croc}[1]{\left\{#1\right\}}
\newcommand\amax{\ensuremath{A_{\max}}\xspace}

\title{Reinforcement Learning in a Birth and Death Process: Breaking the Dependence on the State Space}

\author{%
  {Jonatha Anselmi} \\
  \texttt{jonatha.anselmi@inria.fr} \\
  {Univ.\ Grenoble Alpes, Inria, CNRS, Grenoble INP, LIG, 38000 Grenoble, {France.}}\\
  \And
  {Bruno Gaujal} \\
  \texttt{bruno.gaujal@inria.fr}\\
  {Univ.\ Grenoble Alpes, Inria, CNRS, Grenoble INP, LIG, 38000 Grenoble, {France.}}\\
  \And
{Louis-S\'ebastien Rebuffi}\\
\texttt{louis-sebastien.rebuffi@univ-grenoble-alpes.fr}\\
{Univ.\ Grenoble Alpes, Inria, CNRS, Grenoble INP, LIG, 38000 Grenoble, {France.}}
}

\begin{document}

\maketitle

\begin{abstract}

In this paper, we revisit  the regret of  undiscounted  reinforcement learning in MDPs  with a birth and death structure. Specifically, we consider a controlled queue  with impatient jobs and the main objective is to optimize a trade-off between energy consumption and user-perceived performance. Within this setting, the \emph{diameter}~$D$ of the MDP is $\Omega(S^S)$, where $S$ is the number of states. Therefore, the existing lower and upper bounds on the regret at time~$T$, of  order ~$O (\sqrt{DSAT})$ for MDPs with $S$ states and $A$ actions, may suggest that reinforcement learning is inefficient here. 
In our main result however, we exploit the structure of our MDPs to show that the regret of a slightly-tweaked version of the classical learning algorithm {\sc Ucrl2} is in fact upper bounded by $\tilde{\mathcal{O}} (\sqrt{E_2AT})$ where $E_2$ is related to the weighted second moment of the stationary measure of a reference policy. Importantly, $E_2$ is bounded independently of~$S$. Thus, our bound is asymptotically independent of the number of states and of the diameter. This result is based on a careful study of the number of visits performed by the learning algorithm to the states of the MDP, which is highly non-uniform.
\end{abstract}

\section{Introduction}

In the context of undiscounted reinforcement learning in Markov decision processes (MDPs), it has been shown in the seminal work~\cite{jaksch-2010} that the total regret of any learning algorithm with respect to an optimal policy is lower bounded by $\Omega (\sqrt{DSAT})$, where $S$ is the number of states, $A$ the number of actions, $T$ the time horizon and $D$ the \emph{diameter} of the MDP. Roughly speaking, the diameter is the mean time to move from any state~$s$ to any other state~$s'$ within an appropriate policy.
In the literature, several efforts have been dedicated to approach this lower bound. As a result, learning algorithms have been developed with a total regret of~$\tilde{\mathcal{O}}(DS\sqrt{AT})$ in~\cite{jaksch-2010}, $\tilde{\mathcal{O}}(D\sqrt{SAT})$ in~\cite{azar-2017}  and even~$\tilde{\mathcal{O}}(\sqrt{DSAT})$ according to ~\cite{Tossou-2019,Zhang-2019}. 
These results may give a sense of optimality since the lower bound is attained up to some universal constant. 
However, lower bounds are based on the minimax approach, which relies on the worst possible MDP with given~$D$, $A$ and $S$. This means that when a reinforcement learning algorithm is used  on a given MDP, one can expect a much better performance.

One way to alleviate the minimax lower bound is to consider \emph{structured reinforcement learning}, or equivalently MDPs with some specific structure. The exploitation of such structure may yield more efficient learning algorithms or tighter regret analyses of existing learning algorithms.
In this context, a first example is to consider \emph{factored} MDPs~\cite{FMDPs,guestrinEfficientSolutionAlgorithms2003}, i.e., MDPs where the state space can be factored into a number of components; in this case, roughly speaking, $S=K^n$ where~$n$ is the number of ``factors'' and $K$ is the number of states in each factor.
The regret of learning algorithms in factored MDPs
has been analyzed in \cite{tianMinimaxOptimalReinforcement2020,rosenbergOracleEfficientReinforcementLearning2020,xuReinforcementLearningFactored2020a,osbandNearoptimalReinforcementLearning2014a} and it is found that the $S$ term of existing upper bounds can be replaced by~$nK$.
A similar approach is used in~\cite{Khun-2021} to learn the optimal policy in stochastic bandits with a regret that is logarithmic in the number of states. 
There is also a line of research works that exploit the parametric nature of MDPs.
Inspired by parametric bandits, a $d$-linear additive model was introduced in~\cite{Jin-2019}, where it is shown that an optimistic modification of Least-Squares Value Iteration, see~\cite{LSVI}, achieves a  regret over a finite horizon $H$ of $\tilde{\mathcal{O}}(\sqrt{d^3H^3T})$ where
$d$ is the ambient dimension of the feature space (the number of unknown parameters). In this case, the regret does not depend  on the number of states and actions and the diameter is replaced by the horizon.
A discussion about the inapplicability of this approach to our case is postponed to Section~\ref{ssec:compa}.

\paragraph{Learning in Queueing Systems.}

The control of queueing systems is undoubtedly one of the main application areas of MDPs; see, e.g.,~\cite[Chapters~1--3]{Puterman-1994} and \cite{Li2019}.
Within the rich literature of structured reinforcement learning however, 
few papers are dedicated to reinforcement learning in queueing systems, see \cite[Section~5]{Walton21},
and this motivates us to examine the total regret in this context.
Typical control problems on queues have the following distinguishing characteristics:

\begin{enumerate}[leftmargin=*]
\item
{\it No discount}. Discounting costs or rewards is common practice in the reinforcement learning literature, especially in Q-learning algorithms~\cite{Sutton-1998}. 
However, in queues one is typically interested in optimizing with respect to the average cost.

\item 
{\it Large diameter.} 
Queueing systems are usually investigated under a drift condition that makes the system \emph{stable}, i.e., positive recurrent.
This condition implies that some states are hard to reach.
In fact, for many queueing control problems, the diameter $D$ is exponential in the size of the state space.
Even in the simple case of an M/M/1 queue with a finite buffer, or equivalently a birth--death process with a finite state space and constant birth and death rates, the diameter is exponential in the size of the state space. 

\item
{\it Structured transition matrices.} 
Queueing models describe how jobs join and leave queues, and this yields bounded state transitions.
As a result, MDPs on queues have sparse and structured transition matrices. 
\end{enumerate}

The regret bounds discussed above and item 2 may suggest that the total regret of existing learning algorithm, when applied to queueing systems, is large. 
However, they often work  well in practice and this bring us to consider the following question: 
\emph{
When the underlying MDP has the structure of a queueing system,
do the diameter~$D$ or the number of states $S$ actually play a role in the  regret? }

\paragraph{Our Contribution.} In this paper, we examine the previous question with respect to the class of control problems presented in~\cite{anselmi2021}.
Specifically, an infinite sequence of jobs joins a service system over time to receive some processing according to the first-come first-served scheduling rule; the system can buffer at most $S-1$ jobs and in fact it corresponds to an M/M/1/S-1 queue.
In addition, each job comes with a deadline constraint, and if a job is not completed before its deadline, then it becomes obsolete and is removed from the system. The controller  chooses  the server processing speed and the objective is to design a speed policy for the server that minimizes its average energy consumption plus an obsolescence cost per deadline miss. Although this may look quite specific, this problem captures the typical characteristics of a controlled queue: i) the transition matrix has the structure of a birth and death process with jump probabilities that are affine functions of the state and ii) the reward is linear in the state and convex in the action.
For any MDP in this class, defined in full details in Section~\ref{sec:ourclass},
we show that the diameter is $D=\Omega(S^{S-2})$; see Appendix \ref{app:MDP}.
Thus, without exploiting the particular structure of this MDP, the existing lower and upper bounds
do not justify the reason why standard learning algorithms work efficiently here.

We provide a slight variation of the learning algorithm {\sc Ucrl2}, introduced in~\cite{jaksch-2010},
and show in our main result that the resulting regret is upper bounded by $\tilde{\mathcal{O}}(\sqrt{E_2AT})$ where $E_2$ is a term that depends on the stationary measure of a reference policy defined in Section~\ref{ssec:prop}.
Importantly, $E_2$ does not depend on~$S$.
Thus, efficient reinforcement learning can be achieved independently 
of the number of states by exploiting the stationary structure of the MDP.
Let us provide some intuition about our result.
First, one may think that any learning algorithm should visit each state a sufficient number of times, which justifies why the diameter of an MDP appears in existing regret analyses.
However, this point of view does not take into account the fact that the value of an MDP is the scalar product of the reward and the {stationary measure} of the optimal policy.
If this stationary measure is ``highly non-uniform'', then some states are rarely visited under the optimal policy and barely contribute to the value. In this case, we claim that the learner may not need to visit the rare states that often to get a good estimation of the value, and thus it may not need to pay for the diameter.

\section{Reinforcement Learning Framework}\label{sec:RL}

We consider a unichain Markov Decision Process (MDP) $M=(\mathcal{S},\mathcal{A},P,r)$ in discrete time where~$\mathcal{S}$ is the finite state space,~$\mathcal{A}$ the finite action space,~$P$ the transition probabilities and~$r$ the expected reward function~\cite{Puterman-1994}.
Let also $S:=|\mathcal{S}|$ and  $A:=|\mathcal{A}|$ where $|\cdot|$ is the set cardinality operator.
The model-based reinforcement learning problem consists in finding a \emph{learning algorithm}, or learner, that chooses actions to maximize a cumulative reward over a finite time horizon~$T$.
At each time step~$t\in\mathbb{N}$, the system is in state~$s_t\in\mathcal{S}$ and the learner chooses an action~$a_t\in \mathcal{A}$. When executing~$a_t$, the learner receives a random reward $r_t(s_t,a_t)$ with mean~$r(s_t,a_t)$ and the system moves, at time step $t+1$, to state $s'$ with probability $P(s'|s_t,a_t)$.
The learning algorithm does not know the MDP~$M$ except for the sets~$\mathcal{S}$ and~$\mathcal{A}$.

For simplicity, in the following we consider \emph{weakly communicating} MDPs.
Since we will be interested in the long-run average cost, this will let us remove the dependence on the initial state for several quantities of interest.

\subsection{Undiscounted Regret}
\label{sec:2.1}

Given an MDP $M$, 
let $\Pi:=\{\pi:\mathcal{S}\to\mathcal{A}\}$ denote the set of {stationary and deterministic policies} and let
\begin{align}
\label{rho_M_pi}
\rho(M,\pi) :=  \lim_{T\to\infty} \frac{1}{T}\sum_{t=1}^T \mathbb{E}[r(s_t,\pi(s_t))]
\end{align}
denote the average reward induced by policy $\pi$. Since $M$ has finite state and action spaces,
we notice that i) The limit in~\eqref{rho_M_pi} always exists, ii) It  does not depend on the initial state~$s_0$ when~$M$ is unichain~\cite{Puterman-1994}
 and iii) The restriction to stationary and deterministic policies is not a loss of optimality~\cite[Theorem~8.4.5]{Puterman-1994}.

Let also
$\rho^* := \rho^*(M) := \max_{\pi\in \Pi} \rho(M,\pi)$
be the optimal average reward.

\begin{Def}[Regret]
The \emph{regret} at time~$T$ of the learning algorithm $\mathbb{L}$ is
\begin{align}
\label{def:regret}
{\rm Reg}(M,\mathbb{L}, T):= T \rho^*(M)-\sum_{t=1}^T r_t.
\end{align}
\end{Def}
The regret~\ref{def:regret} is a natural benchmark for evaluating the performance of a learning algorithm.
In~\cite{jaksch-2010}, a \emph{universal} lower bound on ${\rm Reg}(M,\mathbb{L},T)$ has been developed  in terms of the \emph{diameter} of the underlying MDP.
\begin{Def}[Diameter of an MDP]
Let $\pi: \cal{S} \to \cal{A}$ be a stationary policy of $M$ with initial state $s$.
Let $T\left(s'|M,\pi,s\right):=\min\{t\ge 0: s_t=s'|s_0=s\}$ be the random variable for the first time step in which $s'$ is reached from $s$ under $\pi$. 
Then, we say that the \emph{diameter} of~$M$ is
$$
D(M):=\max_{s\neq s'} \min_{\pi:\cal{S}\to \cal{A}} \mathbb{E}\left[ T\left(s'|M,\pi,s\right)\right].
$$
\end{Def}
It should be clear that the diameter of an MDP can be  large if there exist states that are hard to reach. Within the set of structured MDPs considered in this paper, this will be the case and we will show that $D=\Omega(S^{S-2})$.
The following result shows that all learning algorithms have a regret that eventually increase with $\sqrt{D}$.

\begin{The}[Universal lower bound~\cite{jaksch-2010}]
\label{th:ULB}
For any learning algorithm $\mathbb{L}$, any natural numbers $S, A \ge  10$, $D \ge 20 \log_A S$, and $T \ge DSA$,
there is an MDP $M$ with $S$ states, $A$ actions, and diameter $D$ such that for any initial state $s \in \mathcal{S}$,
\begin{align}
\mathbb{E} [{\rm Reg}(M,\mathbb{L}, T)] \ge  0.015 \sqrt{DSAT}.
\end{align}
\end{The}


In view of this result, the diameter of an MDP and its state space appear to be  critical parameters for evaluating the performance of a learning algorithm.

\subsection{The {\sc Ucrl2} Algorithm}

We now focus on  {\sc Ucrl2}, a classical reinforcement learning algorithm introduced in~\cite{jaksch-2010} that is a variant of {\sc Ucrl}~\cite{UCRL}.
While more efficient algorithms have been proposed for the general case (see for example \cite{azar-2017,Tossou-2019}), we will show  that  {\sc Ucrl2} already achieves a very low regret, namely $\tilde{\mathcal{O}}(\sqrt{AT})$, independent of $S$ so using more refined algorithms can only bring marginal gains.

{\sc Ucrl2} is based on \emph{episodes}.
For each episode~$k$, let~$t_k$ denote its start time.
For each state $s$ and action $a$, let $\nu_k(s,a)$ denote the number of visits of~$(s,a)$ during episode~$k$ and let~$N_t(s,a) :=\#\{\tau<t:s_\tau=s,a_\tau=a\}$ denote the number of visits of~$(s,a)$ until timestep~$t$.
Let $\mathcal{M}_k$ be the confidence set of MDPs with transition probabilities $\tilde{p}$ and rewards $\tilde{r}$ that are ``close'' to the empirical MDP at episode $k$,  $\hat{p}_k$ and $\hat{r}_k$, i.e., $\tilde{p}$ and $\tilde{r}$ satisfy
\begin{equation}
 \forall (s,a), \quad \left|\tilde{r}(s,a)-\hat{r}_k(s,a) \right|\leq r_{\max} \sqrt{\frac{7\log\left(2SAt_k/\delta \right)}{2\max\croc{1,N_{t_k}(s,a)}}}
 \label{confR}
\end{equation}
\begin{equation}
\forall (s,a), \quad \left\| \tilde{p}(\cdot|s,a)-\hat{p}_{k}(\cdot|s,a)\right\|_1\leq \sqrt{\frac{14S\log\prt{2 At_k/\delta}}{\max\{1,N_{t_k}(s,a)\}}}.
\label{confP}
\end{equation}
With these quantities, a pseudocode for {\sc Ucrl2} is given in Algorithm~\ref{algo:UCRL2}. 
We notice that {\sc Ucrl2} relies on Extended Value Iteration (EVI), that is a variant of the celebrated Value Iteration (VI) algorithm~\cite{Puterman-1994}; for further details about EVI, we point the reader to~\cite[Section~3.1]{jaksch-2010}.
Let us comment on how {\sc Ucrl2} works.
\begin{algorithm}
    \label{algo:UCRL2}
  \caption{The {\sc Ucrl2} algorithm.}
  \KwIn{A confidence parameter $\delta\in(0,1)$, $\mathcal{S}$ and $\mathcal{A}$.}
  \KwOut{.}
  Set $t:=1$ and observe $s_1$\\
  \For{ episodes $k=1,2,\ldots$ }{
  

    Compute the estimates $\hat{r}(s,a)$ and $\hat{p}_k(s'|s,a)$ as in~\eqref{eq:estimates_r_p}.


    Use ``Extended Value Iteration'' to find a policy $\tilde{\pi}_k$ and an optimistic MDP $\tilde{M}_k \in \mathcal{M}_k$ such that
\begin{equation}
\label{ucrl2_ineq}
%
\rho (\tilde{M}_k, \tilde{\pi}_k)\ge \max_{M'\in {\mathcal{M}}_k,\pi} \rho (M', {\pi}) - \frac{1}{\sqrt{t_k}} 
\end{equation}


  \While{$\nu_k(s_t,\tilde{\pi}_k(s_t)) < \max \{1, N_{t_k}(s_t,\tilde{\pi}_k(s_t))\}$}{

    Choose action $a_t = \tilde{\pi}_k(s_t)$, obtain reward $r_t$ and observe $s_{t+1}$;
    
    $\nu_k(s_t,a_t):=\nu_k(s_t,a_t)+1$;

    $t:=t+1$;
  }    
}
\end{algorithm}
There are three main steps.
First, at the start of each episode, {\sc Ucrl2} 
computes the empirical estimates
\begin{equation}
\label{eq:estimates_r_p}
\hat{r}_k(s,a):=
\frac
{\sum_{t=1}^{t_k-1} r_t \mathds{1}_{\{s_t=s,a_t=a\}}}
{\max\left\{1,N_{t_k}(s,a)\right\}},
\qquad 
\hat{p}_k(s'|s,a):=
\frac
{\sum_{t=1}^{t_k-1} \mathds{1}_{\{s_t=s,a_t=a, s_{t+1}=s'\}}}
{\max\left\{1,N_{t_k}(s,a)\right\}}
\end{equation}
of the reward and probability transitions, respectively,
where $\mathds{1}_{E}$ is the indicator function of~$E$.
Then, 
it applies Extended Value Iteration (EVI) 
to find a policy $\tilde{\pi}_k$ and an optimistic MDP $\tilde{M}_k \in \mathcal{M}_k$ such that~\eqref{ucrl2_ineq} holds true.
Finally, it executes policy $\tilde{\pi}_k$ until it finds a state-action pair $(s,a)$ whose count within episode~$k$ is greater than the corresponding state-action count prior to episode~$k$.


\section{Controlled Birth and Death Processes for Energy Minimization}
\label{sec:ourclass}

Now, we focus on a specific class of MDPs that has been introduced in~\cite{anselmi2021}, which provides a rather general example of a controlled birth and death process with convex costs on the actions and linear rates.
We will denote by $\mathcal{M}$ the set of MDPs with the structure described below.
The MDPs in~$\mathcal{M}$ have been proposed to represent a Dynamic Voltage and Frequency Scaling (DVFS) processor executing jobs with soft obsolescence deadlines. Here, jobs arrive according to a Poisson process with rate $\lambda \in [0,\lambda_{\max}]$ in a buffer of size~$S-1$.
If the buffer is full and a job arrives, then the job is rejected.
Each job has a deadline and a size, i.e., amount of work, which are exponentially distributed random variables with rates $\mu \in [0,\mu_{\max}]$ and one, respectively. Job deadlines and sizes are all independent random variables.
If a job misses its deadline, which is a real time constraint activated at the moment of its arrival, it is removed from the queue without being served and a cost~$C$ is paid.
The processor serves jobs 
under any work-conserving scheduling discipline, e.g., first-come first-served,
with a processing speed that belongs to the finite set $\{0,\ldots, \amax \}$. The objective is to design a speed policy that minimizes the sum of the long term power dissipation and the cost induced by jobs missing their deadlines. When the processor works at speed $a \in \{0,\ldots , \amax \}$, it processes $a$ units of work per second while its power dissipation is $w(a)$.

After uniformization, it is shown in~\cite{anselmi2021} that this control problem can be modeled as an MDP in discrete time with a ``birth-and-death'' transition matrix of size~$S$.
Specifically, we have an MDP $M=(\mathcal{S},\mathcal{A},P,r)$ where $\mathcal{S}=\{0, \dots, S-1\}$, with $s\in\mathcal{S}$ representing the number of jobs in the system, and $\mathcal{A}=\{0, \ldots,\amax\}$, with $a\in\mathcal{A}$ representing the processor speed.
Then, the transition probabilities under policy $\pi$ are given by
\[P_{i,j}(\pi)=
\begin{cases}
 \frac{1}{U} \lambda_i & \text{ if }   i < S-1 \text{ and } j=i+1  \\ 
 \frac{1}{U}(\pi(i)+ i \mu) & \text{ if }  i > 0  \text{ and } j= i-1 \\ 
 P_{ii} & \text{ if } j=i \\
0 & \text{ otherwise},
\end{cases}\]
where $U:=\lambda_{\max} + (S-1) \mu_{\max} + \amax$ is a uniformization constant,
$P_{ii} = \frac{1}{U} (U-\lambda_i -\mu i-\pi(i))$ and
$\lambda_i :=\lambda \left(1- \frac{i}{S-1}\right)$ is the \emph{decaying} arrival rate.
We have replaced the constant arrival rate $\lambda$ by a decaying arrival rate $\lambda_i$ because we want to learn an optimal policy that does not exploit the buffer size~$S-1$; see~\cite{anselmi2021} for further details.
For conciseness,
Figure~\ref{fig:MM} displays the transition diagram of the Markov chain induced by policy~$\pi$.
\begin{figure}[hbtp]
\centering
\begin{tikzpicture}
\tikzset{state/.style={minimum width=1.1cm,line width=0.3mm,  draw,circle}}
\node[state](0) at (1,0) {$0$};
\draw (2,0) node[]{$\bullet \bullet \bullet$};
\node[state] (i-1) at (3,0) {$i-1$};
\node[state] (i) at (6,0) {$i$};
\node[state] at (9,0) (i+1) {$i+1$};
\draw[every loop,line width=0.3mm]
(i) edge[bend left,auto=left] node {$\frac{1}{U}(i \mu + \pi(i))$} (i-1)
(i) edge[bend  left, auto=left]         node[above]{$\frac{\lambda_i}{U}$} (i+1)
(i) edge[loop]         node[above] {$P_{ii}$} (i);
\draw (10,0) node[]{$\bullet \bullet \bullet$};
\node[state] (i) at (11,0) {$S-1$};
\end{tikzpicture}
\caption{Transition diagram of the Markov chain induced by policy $\pi$ of an MDP in~$\mathcal{M}$.\label{fig:MM}}
\end{figure}
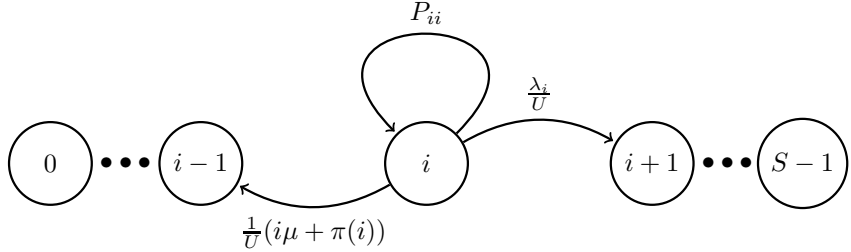

Finally, the reward is a combination of $C$, the constant cost due to a departing job missing its deadline and  $w(a)$,  an arbitrary convex function of $a$, giving the energy cost for using speed $a$.
The immediate cost $c(s,a)$ in state $s$ under action $a$ is a random variable whose value is
$w(a)/U + C $ with probability $s\mu/U$ (missed deadline) and $w(a)/U$ otherwise.
To keep in line with the use of rewards instead of costs, we introduce
a bound on the cost,  $r_{\max} := C + w(\amax)/\mu$ so that the reward in state $s$ under action $a$ is a positive and bounded random variable given by
\begin{equation}
  \label{eq:cost}
  r(s,a):= r_{\max} - c(s,a).
\end{equation}

As in Section~\ref{sec:2.1}, $\rho^*(M)$ is the optimal average cost and $\rho(M,\pi)$ is the average cost induced by policy~$\pi$, where $\pi$ belongs to the set of deterministic and stationary policies $\Pi$.
Since the underlying Markov chain induced by any policy is ergodic, we observe that
\begin{equation}
\rho(M,\pi) = \sum_{s=0}^{S-1} \mathbb{E}[r(s,\pi(s))] m_s^\pi,
\label{eq:val}
\end{equation}
where $m^\pi$ is the stationary measure under policy $\pi$.
In \cite{anselmi2021}, it has been shown that the optimal policy is unique and will be denoted by~$\pi^*$.

\subsection{Properties of \texorpdfstring{${\cal M}$}{M}}\label{ssec:prop}

In the following, we  will use the ``reference'' (or bounding)  policy~$\pi^0(s)=0$ for all $s\in\mathcal{S}$, which thus assigns speed $0$ to all states.
This policy provides a stochastic bound on all policies in the following sense.
Let $s_t^\pi$ be the state under policy~$\pi$ and let $\le_{st}$ denote the \emph{stochastic order}~\cite{shantikumar}; given two random variables $X$ and $Y$ on $\mathbb{R}_+$, we recall that $X\le_{st} Y$ if $\Pb(X\ge s)\le \Pb(Y\ge s)$ for all $s$.

\begin{Lem}
Consider an MDP in $\mathcal{M}$.
For all $t$ and policy~$\pi\in\Pi$, $s_t^\pi \leq_{st} s_t^{\pi^0}$, provided that $s_0^\pi \leq_{st} s_0^{\pi^0}$.
\end{Lem}
\begin{proof}(sketch)
The proof follows by a simple coupling argument between the two policies.
Roughly speaking, each time the Markov chain under $\pi$ decreases from $s$ to $s-1$ because of the speed $\pi(s)$, it stays in state $s$ under policy $\pi^0$.
\end{proof}

Therefore, $\Pb(s_t^\pi \geq s) \leq \Pb(s_t^{\pi^0} \geq s)$ for all $s$ and $t$, which also implies that the respective stationary measures are comparable, i.e., $\sum_{i=s}^{S-1} m^\pi_i  \leq \sum_{i=s}^{S-1} m^{\pi^0}_i$.

Let us now consider $H(s)$, the \emph{bias} at state $s$ of a policy $\pi$, defined 
by
\begin{equation}
H(s):=\mathbb{E}_{\pi}\left[ \sum_{t=1}^\infty \left( r\left(s_t^{\pi},\pi(s_t^{\pi})\right) -\rho^\pi(M)\right) \mid s_0^{\pi}=s \right],\quad \forall 0 \leq s \leq S-1.
\end{equation}
Let also $\partial H(s):=H(s)-H(s-1)$ be the local variation of the bias.

The following result on the variations of the bias holds (see Appendix~\ref{ssec:biasvar}):
\begin{Lem}
  \label{lem:bias}
The local variation of the bias $\partial H(s)$  is bounded: $|\partial H(s) | \leq \Delta(s)$ with $ \Delta(s) := 2r_{\max} e^{\lambda/\mu} (1+ \log s)$ for all $1\leq s \leq S-1$.
\end{Lem}

Both $m^{\pi^0}$ and $\Delta$ will play a major role in our analysis of the regret.

\subsection{Applying  {\sc Ucrl2}  in \texorpdfstring{${\cal M}$}{M}}

We assume that the bounds $\lambda_{\max}$ and  $\mu_{\max}$ are fixed so that $r_{\max}$ is known to the learner. This is a classical assumption, often  replaced by assuming that rewards live in~$[0,1]$.

In the remainder, we will apply {\sc Ucrl2} over an MDP in ${\cal M}$ with a change in the confidence bounds to take into account the support of $P$.
The confidence bounds in \eqref{confR} (resp. \eqref{confP}) are  replaced  by $r_{\max}\sqrt{\frac{2\log\left(2At_k \right)}{\max\croc{1,N_{t_k}(s,a)}}}$ (resp.  $\sqrt{\frac{8\log\prt{2At_k}}{\max\{1,N_{t_k}(s,a)\}}}$).
We also impose that the confidence set ${\cal M}_k$  only contains matrices with the same support as $P$.
Removing $S$ in the confidence bounds does help to reduce the regret. However, by using  existing analysis, this only removes a factor $\sqrt{S}$ in the regret bound (for example, see \cite{UCBVI}).

Finally, note that  {\sc Ucrl2} does not benefit from  the parametric nature of the MDPs in ${\cal M}$,
which is essentially defined by three parameters ($\lambda$, $\mu$ and $C$) and the real convex function~$w(\cdot)$.


\section{Regret of {\sc Ucrl2} on \texorpdfstring{${\cal M}$}{M}}

Our objective is to develop an upper bound on the regret of the learning algorithm {\sc Ucrl2} when applied to MDPs in our class~$\mathcal{M}$. The driving idea is to construct a bound that exploits the structure of the stationary measure of all policies, as they all make some states hard to reach, and to control the number of visits to these states to get a new type of bound.



\subsection{Main Result}

The following theorem gives an upper bound on the regret that does not depend on the classical parameters such as the size of the state space nor on global quantities such as the diameter of the MDP nor the span of the bias of some policy.
Instead, the regret bound below mainly depends on the weighted second moment of the stationary measure of the reference policy $\pi^0$, which is bounded independently of the size of the state space.

We consider the policy $\pi^{\max}$ such that  $\pi^{\max}(s)=A_{\max}$ for all $s$ and $m^{\max}$ its  stationary measure.

Let us also recall that $m^{\pi^0}$ is  the stationary measure of the Markov chain under policy $\pi^0(s)=0$ for all $s$ and that $\Delta: \mathcal{S}\to\R^+$  is a function bounding  the local variations of the optimal bias.
Let $E_2:=F\;\E_{m^{\pi^0}}\brac{(\Delta + r_{\max})^2\cdot f}$ with $f:s\mapsto \frac{\max\{1,s(s-1)\}}{(\Delta(s+1) + r_{\max})^2}$ and $F:=\sum_{s \in \mathcal{S}} f(s)^{-1}$.
Here, $E_2$ is closely related to the second moment of the measure $m^{\pi^0}$ weighted by the bias variations and the maximal reward.

\begin{The}
  \label{th:main_th}
  Let $M\in {\cal M}$. Define $Q_{\max} := \left(\frac{10D}{m^{\max}(S-1)}\right)^2\log \left(\left(\frac{10D}{m^{\max}(S-1)}\right)^4\right)$.
\begin{multline}\label{eq:ourregret}
  \E\brac{{\rm Reg}(M,{\rm  UCRL2},T)} \leq  19\sqrt{ E_2AT \log\left( 2AT\right)} \\
  + 97 {r_{\max}D^2 SA}\max \{ Q_{\max} , T^{1/4}\} \log^2(2AT).
\end{multline}
Here,  $E_2  \leq  60e^{2\lambda/\mu}\,r_{\max}^2\prt{1+\frac{\lambda^2}{\mu^2}}$,
so that the regret satisfies
\[\E\brac{{\rm Reg}(M,{\rm  UCRL2},T)} = \mathcal{O}  \prt{ r_{\max}e^{\lambda/\mu} \sqrt{AT \prt{1+\frac{\lambda^2}{\mu^2}} \log\left( AT\right)}}.\]
  \label{improved}
\end{The}

Before giving a sketch of the proof, let us comment on the bound \eqref{eq:ourregret}. Although the first term is of order $\sqrt{T}$ with a multiplicative constant independent of $S$ - as desired -  the second term,  of order $T^{1/4}$, contains very large terms. However, its interest lies in the novel approach used in the proof that uses the stationary behavior of the algorithm.

\subsection{Comparison with Other Bounds}
\label{ssec:compa}

Let us compare our upper bound with the ones existing in the literature, as we  claim that ours is of a different nature.

First, let us compare with the bound given in~\cite{jaksch-2010}, which states that with probability $1-\delta$, ${\rm Reg}(M,T)\leq 34 DS\sqrt{AT \log\prt{\frac{T}{\delta}}}$ for any $T>1$.
For any $M \in {\cal M}$, the diameter grows as $S^S$ (see Appendix \ref{app:MDP}), thus this bound is very loose here.
More recent works have improved this bound by  replacing the term~$D$ by the local diameter of the MDP~\cite{Bourel-2020}. In Appendix \ref{app:MDP}, we show that the local diameter grows again as $S^S$  for any $M \in {\cal M}$, and thus these results do not yield significant improvements.
Other papers show that the diameter can be replaced by the span of the bias, see~\cite{Fruit-2018,Zhang-2019}. This has a big impact because the span of the bias, for any  $M \in {\cal M}$, is linear in~$S$ (instead of~$S^S$  for the diameter); see Appendix \ref{app:MDP}. However, this is still not as good as the bound given in Theorem~\ref{th:main_th}, which is independent of~$S$.

Now, let us compare with existing bounds for \emph{parametric} MDPs, as mentioned in  the introduction.
The $d$-linear additive model, $d<S$, introduced in~\cite{Jin-2019} assumes that $P(\cdot | s,a) = \langle \phi(s,a),\theta(\cdot)\rangle$, where $\phi(s,a)$ is a known feature mapping and $\theta$ is an unknown measure on $\R^d$.
This form of $P(\cdot | s,a)$ implies that the transition kernel is of rank $d$.  Unfortunately, this property does not hold true in birth and death processes. In fact, the kernel of any  $M \in {\cal M}$ has  almost \emph{full} rank under all policies.
The  {\it linear mixture model}  introduced in \cite{Zhou-2020}  assumes instead that $P(s'|s,a)   = \langle \phi(s'|s,a),\theta\rangle$, $\theta \in \R^d$. This is more adapted to our case, which can be (almost) seen as a linear mixture model of dimension~$d=3$. 
%
The  bound on the discounted regret of the  algorithm proposed in~\cite{Zhou-2020} is ${\rm Reg}(M,T) \leq d\sqrt{T}/(1-\gamma)^2$ 
where~$\gamma$ is a discount factor.
In contrast to our work, this regret analysis holds for \emph{discounted} problems, where we remark that 
both the diameter and the span are irrelevant.
On the other hand, both are replaced by a term of the form~$(1-\gamma)^{-2}$, which implies that the previous bound grows to infinity as $\gamma\uparrow 1$. 
More recently, a regret bound of $O(d\sqrt{D T})$ has been proven in \cite{Wu-2022} in the undiscounted case, that is the case considered in our work. However, the algorithm presented in that reference highly depends on the diameter and is unsuitable for MDPs with a birth and death structure.

Finally, our bound depends on the second moment of the stationary measure of a reference policy, i.e.,~$E_2$, which can be bounded independently of the state space size.
This is structurally different from the ones existing in the literature.
We believe that this structure holds as well in a class of MDPs much larger than ${\cal M}$.
In particular, if $m$ is the stationary measure of some bounding/reference policy, and if the critical quantity $\E_{m}\brac{\Delta \cdot f}$ is small for a well chosen function $f$, then the regret of a learning algorithm navigating the MDP should also be small.
A deeper analysis is left as future work.

\subsection{Sketch of the Proof}\label{ssec:Sketch}


Our proof for Theorem~\ref{th:main_th} is technical and is provided in the supplementary material.
In this section, we present the main ideas and its general structure.
It initially relies on the regret analysis of {\sc Ucrl2} developed in~\cite{jaksch-2010}, and the differences are highlighted below. First, we consider the mean rewards and split the regret into episodes to separately treat the cases where the true MDP is in the confidence set of optimistic MDPs $\mathcal{M}_k$ or not. Thus, let $R_k:= \sum_{s,a}\nu_k(s,a)(\rho^*-\overline{r}(s,a))$ denote the regret in episode~$k$.
 This split can be written:
\begin{equation*}
 \E\brac{{\rm Reg}(M,T)}\leq \E\brac{R_{\rm in}}+\E\brac{R_{\rm out}},
\end{equation*}
where $R_{\rm in}:=\sum_k R_k\1_{M\in\mathcal{M}_k}$ and $R_{\rm out}:=\sum_k R_k\1_{M\notin\mathcal{M}_k}$.

To control $R_{\rm out}$, we use, as in  \cite{jaksch-2010}, the stopping criterion and the confidence bounds. This gives $\E\brac{R_{\rm out}}\leq r_{\max}S$, so that the regret due to episodes where the confidence regions fail will be negligible next to the main terms. 
Then,
when the true MDP belongs to the confidence region, we use the properties of Extended Value Iteration (EVI) to decompose $R_{\rm in}$  into 
\begin{equation*}
\underbrace{\sum_{k,s,a} \nu_k(s,a)(\tilde{r}_k-\overline{r}(s,a))}_{R_{\rm rewards}} 
+ \underbrace{\sum_k\mathbf{v}_k\left(\tilde{\mathbf{P}}_k - \mathbf{I} \right) \mathbf{\tilde{h}}_k}_{R_{\rm bias}}
+ \underbrace{\sum_k\mathbf{v}_k\left(\tilde{\mathbf{P}}_k - \mathbf{I} \right) \mathbf{d}_k+ 2r_{\max}\sum_{k,s,a}\frac{\nu_k(s,a)}{\sqrt{t_k}}}_{R_{\rm EVI}},
\end{equation*}
where $\mathbf{v}_k$ is the vector of the state-action counts $\nu_k$'s, $\tilde{\mathbf{P}}_k$ and $\mathbf{\tilde{h}}_k$ are respectively the transition matrix and the bias in $\tilde{M}_k$ under  policy $\tilde{\pi}_k$, and $\mathbf{d}_k$ is the profile difference between the last step of EVI and the bias (see Appendix~\ref{app:Rin}).

We now show how to handle $R_{\rm rewards}$, $R_{\rm EVI}$, $R_{\rm bias}$.
First, we deal with the terms that do not involve the bias. Using the confidence bound on the rewards (see Appendix~\ref{app:Rtrans}:
\begin{equation}
R_{\rm rewards}\leq r_{\max}2\sqrt{2\log(2AT)}\sum_k\sum_{s,a}\frac{\nu_k(s,a)}{\sqrt{\max{\{1,N_{t_k}(s,a)\}}}}.
 \label{eq:term rewards}
\end{equation}
Now, let us consider $R_{\rm EVI}$. Since $\mathbf{d}_k$ becomes arbitrarily small after enough iterations of EVI (see Appendix~\ref{app:EVI}),
 for $T\geq \frac{e^8}{2AT}$, we get
\begin{equation}
R_{\rm EVI}\leq r_{\max}2\sqrt{2\log(2AT)}\sum_k\sum_{s,a}\frac{\nu_k(s,a)}{\sqrt{\max{\{1,N_{t_k}(s,a)\}}}}.
 \label{eq:term EVI}
\end{equation}
The analysis of~$R_{\rm bias}$ is different from the one in~\cite{jaksch-2010}:
While in \cite{jaksch-2010} the bias is directly bounded by the diameter, we can use the variations of the bias to control the regret more precisely. Using $\mathbf{P}_k$ and $\mathbf h_k$, i.e., the transitions and the bias in the true MDP under $\tilde{\pi}_k$, $R_{\rm bias}$ is further decomposed into:
\begin{equation*}
 \underbrace{ \sum_k\mathbf{v}_k\left(\tilde{\mathbf{P}}_k - \mathbf{P}_k \right) \mathbf{h}_k}_{ R_{\rm trans}} 
 +\underbrace{\sum_k\mathbf{v}_k\left(\tilde{\mathbf{P}}_k - \mathbf{P}_k \right) \prt{\mathbf{\tilde{h}}_k-\mathbf{h}_k}}_{R_{\rm diff}}
 + \underbrace{\sum_k\mathbf{v}_k\left(\mathbf{P}_k - \mathbf{I} \right) \mathbf{\tilde{h}}_k}_{R_{\rm ep}}.
\end{equation*}
The term  $R_{\rm ep}$ can be treated in a similar manner as in~\cite{jaksch-2010} by bounding the bias terms with the diameter to apply an Azuma-Hoeffding inequality (see Appendix~\ref{app:Rep}).
Here, we obtain
\begin{equation*}
 \mathbb{E}\brac{R_{\rm ep}} \leq SAD \,r_{\max}\log_2\left(\frac{8T}{SA}\right).
\end{equation*}
Next, we show in \ref{app:Rdiff} that  $R_{\rm diff}$  does  not contribute to the main term of the regret. This is one of the hard point in our proof. First, linear algebra techniques are used to bound   $||\tilde{\mathbf{h}}_k-\mathbf{h}_k||_{\infty}$  by  $D (2r_{\max} D||\tilde{\mathbf{P}}_k-\mathbf{P}_k||_{\infty} + ||\tilde{\mathbf{r}}_k-\mathbf{r}_k||_{\infty})$. Each norm  is then  bounded using  Hoeffding inequality. This introduces the special quantity  $N_{t_k} (x_k,\pi_k(x_k))$ that yields to the worst confidence bound in episode $k$.
Then, an adaptation of McDiarmid's inequality to Markov chains is used to show   that  $N_{t_k} (x_k,\pi_k(x_k)) \geq (t_{k+1} - t_{k})m^{\max}(S-1)/2 $ with high probability, where $m^{\max}(S-1)$ is the stationary measure of state $S-1$ under the uniform policy $\pi^{\max} (s) = A_{\max}$. This eventually implies that 
\begin{equation*}
  \E\brac{R_{\rm diff}}\leq 96 {r_{\max}D^2 SA} \max\{ Q_{\max} , T^{1/4}\}\log^2(2AT),
\end{equation*}
where  $Q_{\max} := \left(\frac{10D}{m^{\max}(S-1)}\right)^2\log \left(\left(\frac{10D}{m^{\max}(S-1)}\right)^4\right)$.

Then, to deal with the main term $R_{\rm trans}$, we exploit the optimal bias. The unit vector being in the kernel of $\tilde{\mathbf{P}}_k - \mathbf{P}_k $, we can rewrite:
$$R_{\rm trans}= \sum_k\sum_s \sum_{s'} \nu_k\left(s,\tilde{\pi}_k(s)\right)\cdot \left(\tilde{p}_k\left(s'|s,\tilde{\pi}_k(s)\right)-p\left(s'|s,\tilde{\pi}_k(s)\right) \right)\cdot \prt{h^*(s')-h^*(s)}$$
and, thus, using the confidence bound and the bounded variations of the bias,
$$R_{\rm trans} \leq 4\sqrt{2 \log\left( 2AT\right)} \sum_k\sum_{s,a} \frac{\Delta(s+1) \nu_k(s,a)}{\sqrt{\max\{1,N_{t_k}(s,a)\}}}.$$

We can now aggregate  $R_{\rm trans}$, $R_{\rm rewards}$ and  $R_{\rm EVI}$ to compute the main term of the regret (see Appendix~\ref{app:Rmain}).
Here, the key ingredient is to bound
\begin{equation*}
 \sum_{k,s,a} \frac{(\Delta(s+1)+r_{\max}) \nu_k(s,a)}{\sqrt{\max\{1,N_{t_k}(s,a)\}}}
 \label{eq:visits}
\end{equation*}
independently of $S$. This is the second main difference with \cite{jaksch-2010}. Instead of  exploring the MDP uniformly, we know that the algorithm will  mostly visit the first states of the MDP, regardless of the chosen policy. As shown in \cite{jaksch-2010}, for a fixed state $s$:
$$ \mathbb{E}\left[\sum_{a}\sum_k \frac{\nu_k(s,a)}{\sqrt{\max\{1,N_{t_k}(s,a)\}}}\right] \leq 3  \sqrt{\mathbb{E}\left[N_T(s)\right]A}.$$
Now, instead of summing over the states, we can use properties of stochastic ordering to compare the mean number of visits of a state with the probability measure $m^{\pi^0}$; here, we strongly rely on the birth and death structure of the MDPs in $\mathcal{M}$. For any non-negative non-decreasing function $f:\mathcal{S}\to \R^+$, we obtain
\begin{equation}
 \mathbb{E}\left[\sum_{s\geq 0} f(s)N_t(s)\right] \leq t\sum_{s\geq 0}f(s)m^{\pi^0}(s).
 \end{equation}
Let us choose $f:s\mapsto \frac{\max\{1,s(s-1)\}}{(\Delta(s+1)+r_{\max})^2}$ and let $F:=\sum_s f(s)^{-1}\leq 60e^{2\lambda/\mu}r_{\max}^2$. Define also  $E_2:=F \E_{m^{\pi^0}}\brac{(\Delta + r_{\max})^2 \cdot f}$. Then,
\begin{equation*}
 \mathbb{E}\left[\sum_k\sum_{s,a} \frac{(\Delta(s)+r_{\max}) \nu_k(s,a)}{\sqrt{\max\{1,N_{t_k}(s,a)\}}}\right] \leq 3\sqrt{E_2AT}.
\end{equation*}
In Appendix~\ref{app:Rmain}, we further show  that $E_2\leq 60e^{2\lambda/\mu}r_{\max}^2\prt{1+\frac{\lambda^2}{\mu^2}}$.
Therefore, for the three main terms, we obtain
\begin{equation}
\E\brac{R_{\rm trans}+R_{\rm rewards}+R_{\rm EVI}}\leq 19\sqrt{E_2AT\log\left( 2AT\right)}
 \label{eq:term main}
\end{equation}
and we conclude our proof by combining all of these terms.

\section{Conclusions}

For learning in a class of birth and death processes,  we have  shown that exploiting the stationary measure  in the analysis of classical learning algorithms  yields a $K \sqrt{T}$  regret, where $K$ only depends on the stationary measure  of the system under a well chosen policy. Thus, the dependence on the size of the state space as well as on the diameter of the MDP or its span disappears. We believe that this type of results can be generalized to other cases such as optimal routing, admission control and allocation problems in queuing systems, as the stationary distribution under all policies is uneven between the states.

\bibliographystyle{plain}
\bibliography{biblio}

\newpage
\appendix
%

The  appendix is organized as follows:
We first provide some insights on extended value iterations useful in our construction of the regret.Then, the detailed proof of theorem \ref{th:main_th} is given with bounds on  the five  terms in our decomposition of the regret. A final  appendix provides technical lemmas about MDPs in ${\cal M}$.

\section{Proof of Theorem \ref{th:main_th}}

\subsection{Extended value iteration}\label{app:EVI}

For each episode $k$,  we use the extended value iteration  algorithm described in \cite{jaksch-2010} to compute $\tilde{\pi}_k$ and $\tilde{M} \in \mathcal{M}_k$, an optimistic policy and MDP. The values we iteratively get are defined in the following way:
\begin{align}
u^{(k)}_0(s)&=0 \nonumber \\
u^{(k)}_{i+1}(s)&=\max_{a \in \mathcal{A}}\left\{\tilde{r}(s,a)+\max_{p(\cdot)\in \mathcal{P}(s,a)}\left\{ \sum_{s \in S}p(s')u^{(k)}_i(s')\right\} \right\},
\end{align}
where $\tilde{r}$ is the maximal reward from~\eqref{confR} and $\mathcal{P}(s,a)$ is the set of probabilities from \eqref{confP}.

Now, from \cite[Theorem~7]{jaksch-2010}, we obtain the following lemma on the iterations of extended value iteration.
\begin{Lem}
For episode $k$, denote by $i$ the last step of extended value iteration, stopped when:
\begin{equation}
 \max_s\{u^{(k)}_{i+1}(s)-u^{(k)}_i(s)\} - \min_s\{u^{(k)}_{i+1}(s)-u^{(k)}_i(s)\} < \frac{r_{\max}}{\sqrt{t_k}}.
\end{equation}
The optimistic MDP $\tilde{M}_k$ and the optimistic policy $\tilde{\pi}_k$ that we choose are so that the gain is $\frac{r_{\max}}{\sqrt{t_k}}-$ close to the optimal gain:
\begin{equation}
 \tilde{\rho}_k:=\min_s \rho(\tilde{M}_k,\tilde{\pi}_k,s) \geq \max_{M' \in \mathcal{M}_k,\pi,s'} \rho(M',\pi,s') - \frac{r_{\max}}{\sqrt{t_k}}.
\end{equation}
\label{EVIgain}
\end{Lem}
Moreover from \cite[Theorem~8.5.6]{Puterman-1994}:
\begin{equation}
\left| u^{(k)}_{i+1}(s) - u^{(k)}_i(s) - \tilde{\rho}_k \right| \leq \frac{r_{\max}}{\sqrt{t_k}},
\label{eq:EVIdiff}
\end{equation}
and when the optimal policy yields an irreducible and aperiodic Markov chain, we have that $\tilde{\rho}_k=\rho(\tilde{M}_k,\tilde{\pi}_k,s)$ for any $s$, so that we can define the bias:
\begin{equation}
 \tilde{h}_k(s_0)=\mathbb{E}_{s_0}\brac{\sum_{t=0}^\infty (\tilde{r}(s_t,a_t)-\tilde{\rho}_k)}.
\end{equation}
By choosing iteration $i$ large enough, from \cite[Equation~8.2.5]{Puterman-1994}, we can also ensure that:
\begin{equation}
 \left|u^{(k)}_i(s)-(i-1) \tilde{\rho}_k-\tilde{h}_k(s) \right| < \frac{r_{\max}}{2\sqrt{t_k}},
\end{equation}
so that we can define the following difference
\begin{equation}
 d_k(s):=\left|u^{(k)}_i(s) -\min_s u^{(k)}_i(s) - \left(\tilde{h}_k(s)-\min_s \tilde{h}_k(s)\right) \right| < \frac{r_{\max}}{\sqrt{t_k}}.
\end{equation}

%
%
%
%
%

%
%

\subsection{Regret when \texorpdfstring{$M$}{M} is out of the confidence bound}\label{app:Rout}

Let us compute $\mathbb{E}[{\rm Reg}],$ the expected regret. We will mainly follow the approach in \cite[Section~4]{jaksch-2010}, with a few tweaks. We start by splitting the regret into a sum over episodes and states.

We remind that $\overline{r}(s,a)$ is the overall mean reward and $N_T(s,a)$ the total count of visits. We also define $R_k(s):= \sum_{a}\nu_k(s,a)(\rho^*-\overline{r}(s,a))$ the regret at episode $k$ induced by state $s$, with $\nu_k(s,a)$ the number of visit of $(s,a)$ during episode $k$.

Let $R_{\rm in}:= \sum_{s}\sum_{k=1}^m R_k(s) \mathds{1}_{M \in \mathcal{M}_k}$ and $R_{\rm out}:= \sum_{s}\sum_{k=1}^m R_k(s) \mathds{1}_{M \notin \mathcal{M}_k}.$ We therefore have the split:
\begin{equation}
 \E\brac{ Reg} \leq \E\brac{R_{\rm in}}+\E\brac{R_{\rm out}}.
  \label{eq:First terms}
\end{equation}
Now, let $\nu_k(s)=\sum_a \nu_k(s,a)$ and denote by $\mathcal{M}(t)$ the set of MDPs $\mathcal{M}_k$ such that $t_k\leq t<t_{k+1}$. For the terms out of the confidence sets, we  have:
\begin{align*}
R_{\rm out}
 &\leq r_{\max}\sum_{s}\sum_{k=1}^m \nu_k(s) \mathds{1}_{M \notin \mathcal{M}_k} \\
 &\leq r_{\max}\sum_{s}\sum_{k=1}^m N_{t_k}(s) \mathds{1}_{M \notin \mathcal{M}_k}  \text{ using the stopping criterion}\\
 &=r_{\max}\sum_{t=1}^T \sum_{s}\sum_{k=1}^m \mathds{1}_{t_k=t}N_t(s) \mathds{1}_{M \notin \mathcal{M}(t)}
 \leq r_{\max}\sum_{t=1}^T \sum_{s}N_t(s) \mathds{1}_{M \notin \mathcal{M}(t)}\\
 &= r_{\max}\sum_{t=1}^T \mathds{1}_{M \notin \mathcal{M}(t)}\sum_{s}N_t(s)
 \leq r_{\max}\sum_{t=1}^T t\mathds{1}_{M \notin \mathcal{M}(t)}.
\end{align*}
Taking the expectations:
\begin{align}
 \E\brac{R_{\rm out}}&\leq r_{\max}\sum_{t=1}^T t\Pb\croc{M\notin \mathcal{M}(t)} \nonumber\\
 &\leq r_{\max}\sum_{t=1}^T \frac{tS}{2 t^3}
 \leq r_{\max}\sum_{t=1}^T \frac{S}{2 t^2} \text{ by Lemma~\ref{Lem:conf outside}}\nonumber\\
&\leq r_{\max}S.
\label{eq:out term}
\end{align}
Thus, we have dealt with the cases where the MDP $M$ did not belong to any confidence set, for some episodes. We now need to deal with the rest.

\subsection{Regret terms when \texorpdfstring{$M$}{M} is in the confidence bound}\label{app:Rin}

We now assume that $M \in \mathcal{M}_k$ and deal with the terms in the confidence bound, so that we can omit the repetitions of the indicator functions. For each episode $k$, let $R_{{\rm in},k} := \sum_s R_k$.

We defined $\tilde{\pi}_k$ the optimistic policy computed at episode $k$, now define $\tilde{P}_k:=\prt{\tilde{p}_k(s'|s,\tilde{\pi}_k(s))}$ the transition matrix of that policy on the optimistic MDP $\tilde{M}_k$. Define also $\mathbf{v}_k:=\prt{\nu_k(s,\tilde{\pi}_k}$ the row vector of visit counts during episode $k$. Following the same steps as in \cite{jaksch-2010}, we get the inequality on the regret of episode $k$, assuming $M \in \mathcal{M}_k$, using Lemma ~\ref{EVIgain}:
\begin{align*}
 R_{{\rm in},k} &= \sum_{s,a} \nu_k(s,a)(\rho^*-\overline{r}(s,a))\\
 &\leq \sum_{s,a} \nu_k(s,a)(\tilde{\rho}_k-\overline{r}(s,a)) + r_{\max}\sum_{s,a}\frac{\nu_k(s,a)}{\sqrt{t_k}}\\
 &= \sum_{s,a} \nu_k(s,a)(\tilde{\rho}_k-\tilde{r}_k(s,a))+\sum_{s,a} \nu_k(s,a)(\tilde{r}_k-\overline{r}(s,a)) + r_{\max}\sum_{s,a}\frac{\nu_k(s,a)}{\sqrt{t_k}}.
\end{align*}
Then with \eqref{eq:EVIdiff} and using the definition of the iterated values from EVI, we have for a given state $s$ and $a_s:=\tilde{\pi}_k(s)$:
$$ \left|\prt{\tilde{\rho}_k-\tilde{r}_k(s,a_s)} - \prt{\sum_{s'}\tilde{p}_k(s'|s,a_s) u^{(k)}_i(s') -u^{(k)}_i(s)} \right| \leq \frac{r_{\max}}{\sqrt{t_k}},$$
so that:
$$
R_{{\rm in},k}\leq \mathbf{v}_k\left(\tilde{\mathbf{P}}_k - \mathbf{I} \right) \mathbf{u}_i +\sum_{s,a} \nu_k(s,a)(\tilde{r}_k-\overline{r}(s,a)) + 2r_{\max}\sum_{s,a}\frac{\nu_k(s,a)}{\sqrt{t_k}}.
$$
Remember that for any state $s$: $\left| d_k(s)\right| \leq \frac{r_{\max}}{\sqrt{t_k}},$
where $\mathbf{\tilde{h}}_k$ is the bias of the average optimal policy for the optimist MDP, and:
$$d_k(s):=\left(u^{(k)}_i(s)-\min_x u^{(k)}_i(x)\right) - \left( \mathbf{\tilde{h}}_k(s)-\min_x \mathbf{\tilde{h}}_k(x)\right).$$

Notice that the unit vector is in the kernel of $\left(\tilde{\mathbf{P}}_k - \mathbf{I} \right)$. Therefore, in the first term, we can replace $ \mathbf{u}_i$ by any translation of it. We get:
$$
\mathbf{v}_k\left(\tilde{\mathbf{P}}_k - \mathbf{I} \right) \mathbf{u}_i =\mathbf{v}_k\left(\tilde{\mathbf{P}}_k - \mathbf{I} \right) \mathbf{\tilde{h}}_k + \mathbf{v}_k\left(\tilde{\mathbf{P}}_k - \mathbf{I} \right) \mathbf{d}_k.
$$

so that:
\begin{multline*}
R_{\rm in}\leq \underbrace{\sum_k\sum_{s,a} \nu_k(s,a)(\tilde{r}_k-\overline{r}(s,a))}_{R_{\rm rewards}}
+ \underbrace{\sum_k\mathbf{v}_k\left(\tilde{\mathbf{P}}_k - \mathbf{I} \right) \mathbf{\tilde{h}}_k}_{R_{\rm bias}} \\
+ \underbrace{\sum_k\mathbf{v}_k\left(\tilde{\mathbf{P}}_k - \mathbf{I} \right) \mathbf{d}_k+ 2r_{\max}\sum_k \sum_{s,a}\frac{\nu_k(s,a)}{\sqrt{t_k}}}_{R_{\rm EVI}}.
\end{multline*}
Then, using the assumption on empirical rewards \eqref{confR}, as $M \in \mathcal{M}_k$, and noticing that $ N_{t_k} \leq t_k$:

\begin{equation}
R_{\rm rewards}\leq r_{\max}2\sqrt{2\log(2AT)}\sum_k\sum_{s,a}\frac{\nu_k(s,a)}{\sqrt{\max{\{1,N_{t_k}(s,a)\}}}}.
 \label{eq:term rewards 1}
\end{equation}

For the term $\mathbf{v}_k\left(\tilde{\mathbf{P}}_k - \mathbf{I} \right) \mathbf{d}_k$, which does not appear in the analysis of~\cite{jaksch-2010}, we obtain
\begin{align*}
\mathbf{v}_k\left(\tilde{\mathbf{P}}_k - \mathbf{I} \right) \mathbf{d}_k&\leq \sum_s \nu_k\left(s,\tilde{\pi}_k(s)\right)\cdot \left\|\tilde{p}_k\left(\cdot|s,\tilde{\pi}_k(s)\right)-\mathds{1}_{s} \right\|_1\cdot \sup_{s'}|d_k(s')| \\
&\leq 2r_{\max}\sum_s \frac{\nu_k\left(s,\tilde{\pi}_k(s)\right)}{\sqrt{t_k}}  
\leq 2r_{\max}\sum_{s,a} \frac{\nu_k\left(s,a\right)}{\sqrt{t_k}} \\
&\leq 2r_{\max}\sum_{s,a}\frac{\nu_k(s,a)}{\sqrt{\max{\{1,N_{t_k}(s,a)\}}}},
\end{align*}

where in the last inequality we used  that $\max\{1,N_{t_k}(s,a)\}\leq t_k \leq T$.
Thus,
for $T\geq \frac{e^2}{2AT}$
the regret term coming from the consequences and approximations of EVI satisfies
\begin{equation}
R_{\rm EVI}\leq r_{\max}2\sqrt{2\log(2AT)}\sum_k\sum_{s,a}\frac{\nu_k(s,a)}{\sqrt{\max{\{1,N_{t_k}(s,a)\}}}}.
 \label{eq:term EVI 1}
\end{equation}

Now, by defining $P_k$ and $\mathbf{h}_k$ the transition matrix and the bias of the optimistic policy $\tilde{\pi}_k$ in the true MDP $M$, we have the following decomposition of the middle term:
\begin{equation*}
 \underbrace{ \sum_k\mathbf{v}_k\left(\tilde{\mathbf{P}}_k - \mathbf{P}_k \right) \mathbf{h}_k}_{ R_{\rm trans}}
 +\underbrace{\sum_k\mathbf{v}_k\left(\tilde{\mathbf{P}}_k - \mathbf{P}_k \right) \prt{\mathbf{\tilde{h}}_k-\mathbf{h}_k}}_{R_{\rm diff}}
 + \underbrace{\sum_k\mathbf{v}_k\left(\mathbf{P}_k - \mathbf{I} \right) \mathbf{\tilde{h}}_k}_{R_{\rm ep}}.
\end{equation*}

Overall:

\begin{multline}
 R_{\rm in} \leq
 \underbrace{\sum_k\mathbf{v}_k\left(\tilde{\mathbf{P}}_k - \mathbf{P}_k \right) \mathbf{h}_k}_{R_{\rm trans}}
 + \underbrace{\sum_k\mathbf{v}_k\left(\tilde{\mathbf{P}}_k - \mathbf{P}_k \right) \prt{\mathbf{\tilde{h}}_k-\mathbf{h}_k}}_{R_{\rm diff}}
 +\underbrace{\sum_k\mathbf{v}_k\left(\mathbf{P}_k - \mathbf{I} \right) \mathbf{\tilde{h}}_k}_{R_{\rm ep}} \nonumber\\
+\underbrace{r_{\max}4\sqrt{2\log(2AT)}\sum_k\sum_{s,a}\frac{\nu_k(s,a)}{\sqrt{\max{\{1,N_{t_k}(s,a)\}}}}}_{R_{\rm EVI}+R_{\rm rewards}}.
\label{eq:multiterms}
\end{multline}

\subsubsection{Bound on  \texorpdfstring{$R_{\rm trans}$}{Rtrans}}  \label{app:Rtrans}
 Let us deal with the first term $R_{\rm trans}$. To bound this term, we will use our knowledge of the bias in the true MDP $\mathbf{h}_k$ and on the control of the difference of the transition matrices, and for the second term we will control the difference of the biases.

Notice that for a fixed state $0\leq s \leq S-1$:
\begin{equation*}
 \sum_{s'}p\prt{s'|s,\tilde{\pi}_k(s)}h_k(s')=\sum_{s'} p\prt{s'|s,\tilde{\pi}_k(s)}\prt{h_k(s')-h_k(s)}+h_k(s).
\end{equation*}
The same is true for $\tilde{p}_k$, and knowing the MDP is a birth and death process:
\begin{align*}
 R_{\rm trans} &=\sum_k\sum_s \sum_{s'}\nu_k\left(s,\tilde{\pi}_k(s)\right)\cdot \left(\tilde{p}_k\left(s'|s,\tilde{\pi}_k(s)\right)-p\left(s'|s,\tilde{\pi}_k(s)\right) \right)\cdot h_k(s') \\
 &=\sum_k\sum_s \sum_{s'} \nu_k\left(s,\tilde{\pi}_k(s)\right)\cdot \left(\tilde{p}_k\left(s'|s,\tilde{\pi}_k(s)\right)-p\left(s'|s,\tilde{\pi}_k(s)\right) \right)\cdot \prt{h_k(s')-h_k(s)}\\
 &\leq \sum_k\sum_s \nu_k\left(s,\tilde{\pi}_k(s)\right)\cdot \left\|\tilde{p}_k\left(\cdot|s,\tilde{\pi}_k(s)\right)-p\left(\cdot|s,\tilde{\pi}_k(s)\right) \right\|_1 \sup_{s'} \partial h_k(s)\\
 &\leq 4\sqrt{2 \log\left( 2AT\right)} \sum_k\sum_{s,a} \frac{\Delta(s+1) \nu_k(s,a)}{\sqrt{\max\{1,N_{t_k}(s,a)\}}},
 \end{align*}
where in the last inequality, we used the knowledge on the bounded variations of the bias from Lemma~\ref{lem:bias}, and that the optimistic MDP has transitions close to the true transitions.

\subsubsection{Bound on  \texorpdfstring{$R_{\rm diff}$}{Rdiff}}  \label{app:Rdiff}

We now deal with the term involving the difference of bias, $R_{\rm diff}$. For each episode $k$ with policy $\pi_k$, denote by $x_k$ the state such that the confidence bounds are at their worst and denote by $a_k:=\pi_k(x_k)$ the corresponding action used at this state, so that $N_{t_k}(x_k,a_k)$ is minimal. We therefore have that $ \sqrt{\frac{\log\prt{2At_k}}{\max\{1,N_{t_k}(x_k,a_k)\}}}$ is maximal for episode $k$. The true MDP being  within the confidence bounds, with a triangle inequality:
$$\|\tilde P_k-P_k\|_\infty \leq 4\sqrt{\frac{2\log\prt{2At_k}}{\max\{1,N_{t_k}(x_k,a_k)\}}},$$
and 
$$\|\tilde r_k-r_k\|_\infty \leq 2r_{\max}\sqrt{\frac{2\log\prt{2At_k}}{\max\{1,N_{t_k}(x_k,a_k)\}}}.$$

Then using Lemma~\ref{Lem:bias difference}, and noticing that the bias $\mathbf{\tilde{h}}_k$ and the quantity $\|\sum_{t=1}^T \tilde{P}_k^t\tilde{r}_k\|$ is bounded by the same diameter $D$, using the same argument as in \cite{jaksch-2010} (Equation (11)):
\begin{equation}
 \|\mathbf{\tilde{h}}_k-\mathbf{h}_k\|_\infty \leq 12 T_{hit}r_{\max}D\sqrt{\frac{2\log\prt{2At_k}}{\max\{1,N_{t_k}^{ }(x_k,a_k)\}}}.
 \label{eq:biasbias}
\end{equation}
Hence,
\begin{align*}
 R_{\rm diff} &\leq \sum_s \sum_{s'}\nu_k\left(s,\tilde{\pi}_k(s)\right)\cdot \left(\tilde{p}_k\left(s'|s,\tilde{\pi}_k(s)\right)-p\left(s'|s,\tilde{\pi}_k(s)\right) \right)\cdot (\tilde{h}_k(s')-h_k(s')) \\
 &\leq \sum_s \nu_k\left(s,\tilde{\pi}_k(s)\right)\cdot \left\|\tilde{p}_k\left(\cdot|s,\tilde{\pi}_k(s)\right)-p\left(\cdot|s,\tilde{\pi}_k(s)\right) \right\|_1 \|\mathbf{\tilde{h}}_k-\mathbf{h}_k\|_\infty\\
 &\leq 48 D^2r_{\max} \log\prt{2AT} \Sigma,
 \end{align*}
 where in the last inequality we have used \eqref{eq:biasbias} and that by definition of $D$, for $S$ large enough $$T_{hit}:=\inf_{s' \in \mathcal{S}} \sup_{s \in \mathcal{S}} \E_{s}\, \tau^{\pi_k}_{s'}\leq \E_{S-1}\, \tau^{\pi^0}_{0}\leq D,$$ and we called
 \begin{equation*}
  \Sigma:= \sum_{s,a}\sum_k\sum_{t=t_k}^{t_{k+1}-1}\frac{\1_{\{s_t,a_t=s,a\}}}{\sqrt{\max\{1,N_{t_k}(s,a)\}}\sqrt{\max\{1,N_{t_k}(x_k,a_k)\}}}.
 \end{equation*}
By the choice  of $x_k$, $N_{t_k}(x_k,a_k)\leq N_{t_k}(s,a)$ for any state-action pair $(s,a)$, so that we can rewrite, with $I_k:=t_{k+1}-t_k$ the length of episode $k$:

 \begin{equation*}
  \Sigma\leq \sum_{s,a}\sum_k\sum_{t=t_k}^{t_{k+1}-1}\frac{\1_{\{s_t,a_t=s,a\}}}{\max\{1,N_{t_k}(x_k,a_k)\}}\leq \sum_k \frac{I_k}{\max\{1,N_{t_k}(x_k,a_k)\}}.
 \end{equation*}
 

Now define $Q_{\max}:=\left(\frac{10D}{m^{\max}(S-1)}\right)^2 \log \left(\left(\frac{10D}{m^{\max}(S-1)}\right)^4\right)$, and $I(T):=\max\croc{Q_{\max},T^{1/4}}$. We split the sum depending on whether the episodes are shorter than $I(T)$ or not, and call $K_{\leq I}$ the number of such episodes. This yields:
$$\Sigma\leq K_{\leq I} I(T) + \sum_{k,I_k> I(T)}  \frac{I_k}{\max\{1,N_{t_k}(x_k,a_k)\}}.$$
Using the stopping criterion for episodes:
$$\Sigma\leq K_{\leq I} I(T) + \sum_{k,I_k> I(T)}  \frac{I_k}{\max\{1,\nu_k(x_k,a_k)\}}.$$
Now denote by $\mathcal{E}$ the event: $$\mathcal{E}=\croc{\forall k \text{ s.t } I_k>I(T), \;\frac{1}{\max\{1,\nu(x_k,a_k)\}}\leq  \frac{2}{m^{\max}(S-1)I_k}}.$$
By splitting the sum, using the above event, we get:
\begin{align*}
 \Sigma &\leq K_{\leq I} I(T) + \1_\mathcal{E}\sum_{k,I_k>I(T)}\frac{2}{m^{\max}(S-1)}+\1_{\mathcal{\bar{E}}}\sum_{k,I_k> I(T)} I_k \\
 &\leq K_{\leq I} I(T) + \1_\mathcal{E}\prt{K_T-K_{\leq I}}\frac{2}{m^{\max}(S-1)}+\1_{\mathcal{\bar{E}}}T.
\end{align*}
We use Corollary~\ref{Cor:worst count} to get $\Pb\prt{\mathcal{\bar{E}}}\leq \frac{1}{4T}$, so that when taking the expectation:
\begin{align*}
 \E\brac{\Sigma}&\leq \E\brac{K_{\leq I}} I(T) + \E\brac{\prt{K_T-K_{\leq I}}}\frac{2}{m^{\max}(S-1)}+\frac{1}{4}.
 \end{align*}
 Now using Lemma~\ref{Lem:number of episodes}, $SA\geq4$, $I(T) \geq \frac{2}{m^{\max}(S-1)}$ and that $\frac{1}{\log 2}+ \frac{1}{4}\leq 2$:
 \begin{equation*}
 \E\brac{\Sigma}\leq \E\brac{K_T} I(T)+\frac{1}{4}\leq 2SA\log (2AT) I(T).
\end{equation*}


We therefore have that:
\begin{equation}
 \E\brac{R_{\rm diff}}\leq 96r_{\max}SAD^2 I(T)\log^2\prt{2AT} .
 \label{eq:bias term}
\end{equation}

\subsubsection{Bound on the main terms: Exploiting the stochastic ordering}

In Section~\ref{ssec:Sketch} we have shown that:
\begin{equation}
R_{\rm trans} \leq 4\sqrt{2\log\left( 2AT\right)} \sum_{s,a} \frac{\Delta(s+1) \nu_k(s,a)}{\sqrt{\max\{1,N_{t_k}(s,a)\}}}.
\label{eq:trans tempo}
 \end{equation}
 To control this term as well as $R_{\rm EVI}$ \eqref{eq:term EVI 1} and $R_{\rm rewards}$ \eqref{eq:term rewards 1}, we need to control the terms in the sum in a way that does not make the parameters $D$ or $S$ appear, as this will be one of the main contributing terms. To do so, we need to sum over the episodes and take the expectation, so that with Lemma~\ref{Lem:sum integral}, we get:





\begin{align*}
 \mathbb{E}\left[\sum_{s,a}\sum_k \frac{\nu_k(s,a)}{\sqrt{\max\{1,N_{t_k}(s,a)\}}}\right] &\leq 3 \mathbb{E}\left[\sum_{s,a} \sqrt{N_T(s,a)}\right]\\
 &\leq 3 \sum_{s} \sqrt{\mathbb{E}\left[N_T(s)\right]A}  \text{ by Jensen's inequality.}
\end{align*}

We will use the following lemma to carry on the computations:
\begin{Lem}
Let $m^{\pi^0}$ be the stationary measure of the Markov chain under policy $\pi^0$, such that for every state $s$: $\pi^0(s)=0$. Let $f:\mathcal{S}\to \mathbb{R}^+$ be a non-negative non-decreasing function on the state space.
Then for any state $s\in \mathcal{S}$,
\begin{equation}
 \mathbb{E}\left[\sum_{s'\geq s} f(s')N_t(s')\right] \leq t\sum_{s'\geq s}f(s')m^{\pi^0}(s').
 \end{equation}
 \label{Lem:stochastic}
\end{Lem}
\begin{proof}
 Let $s\in \mathcal{S}$. 
 For any state $s'$, define $N_t^{m^{\pi^0},\pi^0}(s')$ the number of visits when the starting state follows the initial distribution $m^{\pi^0}$, and the MDP always executes the policy $\pi^0$ at every timestep instead of the policy determined by the algorithm {\sc Ucrl2}.
 Notice already that for any state $s'$:
 $$ \E\brac{N_t^{m^{\pi^0},\pi^0}(s')}=tm^{\pi^0}(s')$$
 On the other hand, for $x \in \mathcal{S},$ we have the stochastic ordering:
 $$\sum_{s'\geq x}N_t(s') \leq_{st} \sum_{s'\geq x}N_t^{m^{\pi^0},\pi^0}(s'),$$
 so that for any non-negative non-decreasing function $f$, with the convention $f(-1)=0$:
  \begin{equation}
  \begin{cases}
    (f(x)-f(x-1))\sum_{s'\geq x}N_t(s') \leq_{st} (f(x)-f(x-1))\sum_{s'\geq x}N_t^{m^{\pi^0},\pi^0}(s')\\
    f(s-1)\sum_{s'\geq s}N_t(s') \leq_{st} f(s-1)\sum_{s'\geq s}N_t^{m^{\pi^0},\pi^0}(s')
,\end{cases}
\label{eq:stochastic}
\end{equation}
and then summing the equation above  for $s\leq x\leq S-1$ and switching the sums yields:
$$ \sum_{s'\geq s}N_t(s')\sum_{x=s}^{s'}[f(x)-f(x-1)]  \leq_{st} \sum_{s'\geq s} N_t^{m^{\pi^0},\pi^0}(s') \sum_{x=s}^{s'} [f(x)-f(x-1)],$$
which simplifies to:
$$ \sum_{s'\geq s}N_t(s')[f(s')-f(s-1)]  \leq_{st} \sum_{s'\geq s} N_t^{m^{\pi^0},\pi^0}(s') [f(s')-f(s-1)]. $$
Now summing with the second equation in \eqref{eq:stochastic}, we get the following equation:
 $$ \sum_{s'\geq s}N_t(s')f(s')  \leq_{st} \sum_{s'\geq s} N_t^{m^{\pi^0},\pi^0}(s') f(s'). $$
Taking the expectation in this last inequality finishes the proof.
\end{proof}

Now, we can conclude our bound on $R_{{\rm trans}}$. Since
\begin{equation}
 \mathbb{E}\left[\sum_{s,a}\sum_k (\Delta(s+1)+r_{\max})\frac{\nu_k(s,a)}{\sqrt{\max\{1,N_{t_k}(s,a)\}}}\right] \leq 3 \sqrt{A}\sum_{s\geq 0} (\Delta(s+1)+r_{\max})\sqrt{\E\brac{N_T(s)}},
 \label{eq:base}
\end{equation}
let $f$ be a non-negative non-decreasing function over the state space, such that $F:=\sum_{s\geq0} f(s)^{-1}$ exists. Then by concavity:

\begin{align*}
 \sum_{s\geq 0} (\Delta(s+1)+r_{\max})\sqrt{\E\brac{N_T(s)}}  &= F\sum_{s\geq 0}\frac{1}{Ff(s)} \sqrt{f(s)^2(\Delta(s+1)+r_{\max})^2\E\brac{N_T(s)}}\\
 &\leq F\sqrt{\sum_{s\geq 0}\frac{f(s)^2(\Delta(s+1)+r_{\max})^2\E\brac{N_T(s)}}{Ff(s)}} \text{ by concavity }\\
 &= \sqrt{F\sum_{s\geq 0}f(s)(\Delta(s+1)+r_{\max})^2\E\brac{N_T(s)}}\\
 &\leq \sqrt{T F \sum_{s\geq 0}f(s)(\Delta(s+1)+r_{\max})^2 m^{\pi^0}(s)} \text{ using Lemma~\ref{Lem:stochastic}},
\end{align*}

so that overall, \eqref{eq:base} becomes:
\begin{equation}
\mathbb{E}\left[\sum_{s,a}\sum_k \frac{(\Delta(s+1)+r_{\max})\nu_k(s,a)}{\sqrt{\max\{1,N_{t_k}(s,a)\}}}\right] \leq 3\sqrt{ATF} \sqrt{ \sum_{s\geq 0}f(s)(\Delta(s+1)+r_{\max})^2m^{\pi^0}(s)}.
 \label{eq:base2}
\end{equation}

This is the term mainly contributing to the regret.

\subsubsection{Bound on the main terms: Introducing \texorpdfstring{$E_2$}{E2}}  \label{app:Rmain}

Now, using Lemma \ref{Lem:worst measure} which gives the stationary distribution of $m^0$, we can compute the expectation under $m^{\pi^0}$ of a well-chosen function $f$:

\begin{Lem}
Choose the increasing function $f:s\mapsto \frac{\max\{1,s(s-1)\}}{(\Delta(s+1)+r_{\max})^2}$.
Then $F\leq 60e^{2\lambda/\mu}\,r_{\max}^2 $ and $\sum_{s\geq 0}(\Delta(s+1)+r_{\max})^2f(s)m^{\pi^0}(s)=\E_{m^{\pi^0}}\brac{(\Delta+r_{\max})^2 \cdot f}\leq \prt{1+\frac{\lambda^2}{\mu^2}}$, so that:
$$E_2:=F \E_{m^{\pi^0}}\brac{(\Delta+r_{\max})^2 \cdot f} \leq 60e^{2\lambda/\mu}\,r_{\max}^2\prt{1+\frac{\lambda^2}{\mu^2}}.$$
\end{Lem}

\begin{proof}
 For $F$, we obtain:
 \begin{align*}
  F&\leq (2e^{\lambda/\mu}r_{\max})^2 3\prt{\sum_{s=0}^{S-1} \frac{1+\log(s+1)}{\max\{1,s(s-1)\}}}
  \leq 12e^{2\lambda/\mu}\,r_{\max}^2 \prt{3+\sum_{s=2}^{S-1} \frac{1+\log(s+1)}{s(s-1)}}\\
  &\leq 12e^{2\lambda/\mu}\,r_{\max}^2 \prt{4+\sum_{s=2}^{S-2} \frac{\log(1+1/s)}{s}}
    \leq 60e^{2\lambda/\mu}\,r_{\max}^2. 
 \end{align*}

 Using the following computations:
 \begin{align*}
  \sum_{s=2}^{S-1} s(s-1)\binom{S-1}{s} \prt{\frac{\lambda}{(S-1) \mu}}^s
  &=(S-2)(S-1)\sum_{s=2}^S \binom{S-3}{s-2} \prt{\frac{\lambda}{(S-1) \mu}}^s\\
  &=(S-2)(S-1)\prt{\frac{\lambda}{(S-1) \mu}}^2\;\sum_{s=0}^{S-3} \binom{S-3}{s} \prt{\frac{\lambda}{(S-1) \mu}}^s\\
  &=(S-2)(S-1)\prt{\frac{\lambda}{(S-1) \mu}}^2\prt{1+\frac{\lambda}{(S-1)\mu}}^{S-3}\\
  &\leq \prt{\frac{\lambda}{\mu}}^2\prt{1+\frac{\lambda}{(S-1)\mu}}^{S-3},
 \end{align*}
and using that $1+\frac{\lambda}{\mu} \leq \prt{1+\frac{\lambda}{(S-1)\mu}}^{S-1}$,
we get:
$$\prt{1+\frac{\lambda}{(S-1)\mu}}^{S-1}\E_{m^{\pi^0}}\brac{(\Delta+r_{\max})^2 \cdot f}\leq \prt{1+\frac{\lambda^2}{\mu^2}}\prt{1+\frac{\lambda}{(S-1)\mu}}^{S-1}, $$
so that finally
$$\E_{m^{\pi^0}}\brac{(\Delta+r_{\max})^2 \cdot f} \leq \prt{1+\frac{\lambda^2}{\mu^2}},$$
which concludes the proof.
\end{proof}

Finally $\eqref{eq:base2}$ becomes:
\begin{equation}
 \mathbb{E}\left[\sum_{s,a}\sum_k \frac{(\Delta(s+1)+r_{\max}) \nu_k(s,a)}{\sqrt{\max\{1,N_{t_k}(s,a)\}}}\right] \leq 3\sqrt{E_2AT},
\end{equation}
and thus:
\begin{equation}
 \E\brac{R_{\rm trans}+R_{\rm rewards}+R_{\rm EVI}}\leq 12\sqrt{2E_2AT\log\left( 2AT\right)}.
  \label{eq:transitions term}
\end{equation}

In particular:
\begin{equation}
 \E\brac{R_{\rm trans}+R_{\rm rewards}+R_{\rm EVI}}\leq 132e^{\lambda/\mu}r_{\max}\sqrt{\prt{1+\frac{\lambda^2}{\mu^2}}AT\log\left( 2AT\right)}.
  \label{eq:main term}
\end{equation}

\subsubsection{Bound on \texorpdfstring{$R_{\rm ep}$}{Rep}} \label{app:Rep}

It remains to deal with the following regret term:
$$R_{\rm ep}=\sum_k \mathbf{v}_k \left(\mathbf{P}_k - \mathbf{I} \right) \mathbf{\tilde{h}}_k.$$

We will follow the core of the proof from \cite{jaksch-2010}. Define $X_t:= ~\left( p(\cdot | s_t,a_t) - \mathbf{e}_{s_t}\right)\mathbf{\tilde{h}}_{k(t)}\mathds{1}_{M \in \mathcal{M}_{k(t)}}$, where $k(t)$ is the episode containing step $t$ and $\mathbf{e}_i$ the vector with $i$-th coordinate $1$ and $0$ for the other coordinates.
\begin{align*}
 \mathbf{v}_k \left(\mathbf{P}_k - \mathbf{I} \right) \mathbf{\mathbf{\tilde{h}}_k} &= \sum_{t=t_k}^{t_{k+1}-1} X_t + \mathbf{\tilde{h}}_k(s_{t_{k+1}}) -\mathbf{\tilde{h}}_k(s_{t_k})\\
 &\leq \sum_{t=t_k}^{t_{k+1}-1} X_t + Dr_{\max}.
\end{align*}

By summing over the episodes we get:
\begin{equation*}
 R_{\rm ep} \leq \sum_{t=1}^{T} X_t +K_TDr_{\max}.
\end{equation*}

Notice that $\mathbb{E}\left[X_t|s_1,a_1, \dots, s_t, a_t \right]=0$, so that when taking the expectations, only the term in the number of episodes remains.

On the other side, using Lemma~\ref{Lem:number of episodes},  we get when taking the expectation:
\begin{equation*}
 \mathbb{E}\brac{R_{\rm ep}} \leq SA\log_2\left(\frac{8T}{SA}\right)\cdot Dr_{\max}.
\end{equation*}

Assuming $SA\geq 4$, and using $\log(2)\geq \frac{1}{2}$:

\begin{equation}
 \mathbb{E}\brac{R_{\rm ep}} \leq 2r_{\max}SAD\log (2AT).
 \label{eq:ep term}
\end{equation}


We can now gather the expected regret terms when the true MDP is within the confidence bounds. Using  \eqref{eq:bias term}, \eqref{eq:transitions term} and \eqref{eq:ep term}:

 \begin{equation*}
\E\brac{R_{\rm in}}\leq 96r_{\max}SAD^2 I(T)\log^2\prt{2AT} +12\sqrt{2E_2AT \log\left( 2AT\right)}
 +2r_{\max}SAD\log (2AT),
 \end{equation*}

which gives with \eqref{eq:First terms} and \eqref{eq:out term}, assuming that $T\geq S^2$:

\begin{equation*}
 \E\brac{Reg}\leq 97r_{\max}SAD^2 I(T)\log^2\prt{2AT} + 12\sqrt{2E_2AT \log\left( 2AT\right)}. 
\end{equation*}

which finally gives:
\begin{equation*}
 \E\brac{Reg}\leq 97r_{\max}SAD^2 I(T)\log^2\prt{2AT}+ 19\sqrt{ E_2AT \log\left( 2AT\right)}. 
\end{equation*}

\section{Technical Lemmas}

\subsection{Probability of the confidence bounds}\label{app:UCRL2}

This first lemma is from \cite[Lemma~17]{jaksch-2010} and adapted to our confidence bounds.
\begin{Lem}
 For $t>1$, the probability that the MDP $M$ is not within the set of plausible MDPs $\mathcal{M}_t$ is bounded by:
 \begin{equation*}
  \Pb\croc{M \notin \mathcal{M}(t)} < \frac{S}{2t^3}.
 \end{equation*}
 \label{Lem:conf outside}
 \end{Lem}
\begin{proof}
Fix a state action pair $(s,a)$, and $n$ the number of visits of this pair before time $t$. Recall that $\hat{p}$ and $\hat{r}$ are the empirical transition probabilities and rewards from the $n$ observations. Knowing that from each pair, there are at most $3$ transitions, a Weissman's inequality gives for any $\varepsilon_p >0$:
 \begin{equation*}
  \Pb\croc{ \|\hat{p}(\cdot|s,a)-p(\cdot|s,a)\|_1 \geq \varepsilon_p} \leq 6\exp\prt{-\frac{n\varepsilon_p^2}{2}}.
 \end{equation*}
So that for the choice of $\varepsilon_p = \sqrt{\frac{2}{n} \log\prt{16 At^4}}\leq \sqrt{\frac{8}{n}\log\prt{2 At}},$ we get:
$$
 \Pb\croc{ \|\hat{p}(\cdot|s,a)-p(\cdot|s,a)\|_1 \geq \sqrt{\frac{8}{n}\log\prt{2 At}}} \leq \frac{3}{8 At^4}.
$$
We can do similar computations for the confidence on rewards, with a Hoeffding inequality:
\begin{equation*}
  \Pb\croc{ |\hat{r}(s,a)-r(s,a)| \geq \varepsilon_r} \leq 2\exp\prt{-\frac{2n\varepsilon_r^2}{r_{\max}^2}},
 \end{equation*}
and choosing $\varepsilon_r=r_{\max}\sqrt{\frac{1}{2n} \log\prt{16 At^4}}\leq r_{\max}\sqrt{\frac{2}{n}\log\prt{2 At}},$ so that:
$$
\Pb\croc{ |\hat{r}(s,a)-r(s,a)| \geq r_{\max}\sqrt{\frac{2}{n}\log\prt{2 At}}} \leq \frac{1}{8 At^4}.
$$
 
 Now with a union bound for all values of $n \in \{0, 1, \cdots, t-1\}$, we get:
 $$
 \Pb\croc{ \|\hat{p}(\cdot|s,a)-p(\cdot|s,a)\|_1 \geq \sqrt{\frac{8\log\prt{2 At}}{\max\{1,N_t(s,a)\}}}} \leq \frac{3}{8 At^3},
$$
and
$$
\Pb\croc{ |\hat{r}(s,a)-r(s,a)| \geq r_{\max}\sqrt{\frac{2\log\prt{2 At}}{\max\{1,N_t(s,a)\}}}} \leq \frac{1}{8 At^3},
$$
 and finally, when summing over all state-action pairs, $ \Pb\croc{M \notin \mathcal{M}(t)} < \frac{S}{2t^3}$.
\end{proof}

\subsection{Number of visits for an MDP in \texorpdfstring{${\cal M}$}{M}}\label{app:visits}
This lemma is needed in the proof of Lemma~\ref{Lem:McDiarmid-like}.
\begin{Lem}[Azuma-Hoeffding inequality]
Let $X_1,X_2, \dots$ be a martingale difference sequence with $|X_i| \leq RD$ for all $i$ and some $R>0$. Then for all $\varepsilon >0$ and $n \in \N$:
$$
\Pb\croc{\sum_{i=1}^n X_i \geq \varepsilon} \leq \exp\prt{-\frac{\varepsilon^2}{2nDR}}.
$$
\label{Lem:Azuma}
\end{Lem}
The two following lemmas  are proved in \cite[Appendix~C.2 and Appendix~C.3]{jaksch-2010} respectively. Bounding the number of episodes is notably useful to obtain equation \eqref{eq:bias term}.
\begin{Lem}
Denote by $K_t$ the number of episodes up to time $t$, and let $t>SA$. It is bounded by:
 $$K_t\leq SA\log_2\prt{\frac{8t}{SA}}.$$
 \label{Lem:number of episodes}
\end{Lem}
The following lemma is used to simplify regret terms, notably \eqref{eq:trans tempo}.
\begin{Lem}
For any fixed state action pair $(s,a)$ and time $T$, we have:
 $$\sum_{k=1}\frac{\nu_k(s,a)}{\sqrt{\max\{1,N_{t_k}(s,a)\}}} \leq 3\sqrt{N_{T+1}(s,a)},$$
 \label{Lem:sum integral}
\end{Lem}

\subsection{Diameter and Span of MDPs in \texorpdfstring{${\cal M}$}{M}}
\label{app:MDP}

For completeness,  and to support the discussion in Section \ref{ssec:compa}, the section details the behavior of the diameter and the span of MDPs in  ${\cal M}$, as functions of $S$.

Under policy $\pi^0$, it is possible to get en explicit expression  for the stationary distribution of the states.

\begin{Lem}
Under the stationary policy $\pi^0$, the stationary measure $m^{\pi^0}(s)$ of the MDP is given by:   
 \begin{equation*}
  m^{\pi^0}(s)=\frac{\binom{S-1}{s}\prt{\frac{\lambda}{(S-1)\mu}}^s}{\prt{1+\frac{\lambda}{(S-1)\mu}}^{(S-1)}}.
 \end{equation*}
 \label{Lem:worst measure}
\end{Lem}

This lemma is shown in the proof of \cite[Lemma~3.3]{anselmi2021}.

First, it should be clear that under any policy $\pi$, the diameter of the MDP under $\pi$ is extremely large because the probability to move from state $s$ to state $s+1$ is smaller and smaller as $s$ grows.
Actually, this is also true for the local diameter, more precisely  the expected time to go from $s$ to $s+1$ grows extremely fast with $s$.

This discussion is formalized in the following result.

\begin{Lem}
  \label{lem:diam}
  For any $M \in {\cal M}$ and any policy $\pi$, the diameter $D^\pi$ as well as the local diameter $D^\pi(s-1,s)$ grow as $S^{S-2}$.
\end{Lem}

\begin{proof} 
  Under policy $\pi$, the following sequence of inequalities follows from the stochastic comparison with  $\pi^0$ and monotonicity under   $\pi^0$.
  \begin{align*}
    D^\pi \geq \tau^\pi(0,S-1) \geq \tau^{\pi^0}(0,S-1) \geq  \tau^{\pi^0}(S-2,S-1),
  \end{align*}

  where $\tau^\pi(x,y)$ is the expected time to go from $x$ to $y$ under policy $\pi$.

  Now, starting from $S-2$, the Markov chain moves to $S-1$ with probability $p:= \lambda/(U(S-1))$ and the time to reach $S-1$ is equal to $1$ or moves to $S-2$ or $S-3$ with probability $1-p$. Therefore,  $ \tau^{\pi^0}(S-2,S-1)$ is bounded by  $1-p$ times the  return time to $S-2$ in the chain truncated at $S-2$, bounded in turn by the inverse of the stationary measure of state $S-2$ in this chain .
  Using Lemma \ref{Lem:worst measure},
  \begin{align}
    \tau^{\pi^0}(S-2,S-1) & \geq  (1-p) \left( \frac{(S-2)\mu}{\lambda}\right)^{(S-2)}\prt{1+\frac{\lambda}{(S-2)\mu}}^{(S-2)} \\
                        & =  \exp\prt{\frac{\lambda}{\mu}-2}\left(\frac{\mu}{\lambda} \right)^{S-2}  S^{S-2}  (1+ o(1/S)).
  \end{align}

  As for the maximal local diameter, $\max_s D^\pi(s-1,s) \geq  \max_s \tau^{\pi^0}(s-1,s) \geq  \tau^{\pi^0}(S-2,S-1) $ and the same argument as before applies.
  
\end{proof}

Let us now consider  the bias of the optimal policy in $M$. 
From \cite{anselmi2021}, the bias $h^*(s)$ is decreasing  and concave in $s$, with increments bounded by $C$.
Therefore, its span, defined  as $\mbox{span }(h^*) := \max_s h^*(s) - \min_s h^*(s)$, satisfies
\[ (h^*(0) - h^*(1) ) S  \leq \mbox{span }(h^*) \leq  (h^*(S-2) - h^*(S-1)) S \leq  C (S-1). \] 
This implies that the span of the bias behaves as a linear function of  $S$ for all  $M$.

\section{Generic Lemmas on Ergodic MDPs}\label{app:ErgoMDPs}

\subsection{Bound on the bias variations}\label{ssec:biasvar}

In this subsection, we aim to prove Lemma~$\ref{lem:bias}$, which plays a main role in the regret computations. We consider here we are always in the true MDP $M$.

We first show a bound on hitting times that will be used to control the variations of the bias:

\begin{Lem} 
\label{Lem:hit0}
 Let $\pi$ be any policy. Consider the Markov chain with policy $\pi$ and transitions $P$ starting from any state $s$. Denote by $\tau_{s}$ the random time to hit $0$ from state $s$. Then:
 \begin{equation*}
   \mathbb{E}\tau_{s}\leq m^\pi(0)^{-1} \sum_{i=1}^s \frac{U}{\pi(i)+\mu i}.
 \end{equation*}
\end{Lem}
\begin{proof}
 We write the expected hitting time equations. Let $\mathbf e$ be the unit vector. We have the system:

\begin{equation}
 \E \boldsymbol{\tau}= \mathbf{e} + P\; \E\boldsymbol{\tau},
\end{equation}
with the extra equation $\tau_0=0$.

 We will show the result by induction. Call $\mu_i:=\pi(i)+\mu i$. The system gives for $s=S-1$:
\begin{equation*}
 \mathbb{E}\tau_{s} = 1 + \mathbb{E}\tau_{s}\frac{1-\mu_s}{U}+\mathbb{E}\tau_{s-1}\frac{\mu_s}{U},
\end{equation*}
so that:
\begin{equation*}
 \mathbb{E}\tau_s = \frac{U}{\mu_s}+\mathbb{E}\tau_{s-1}.
\end{equation*}

Then with an induction, we want to prove the equation for $s<S-1$:
\begin{equation}
 \mathbb{E}\tau_{s}=\mathbb{E}\tau_{s-1}+\frac{U}{\mu_s}\sum_{s'= s}^{S-1}\prod_{i=s+1}^{s'}\frac{\lambda_{i-1}}{\mu_i},
 \label{eq:hit}
\end{equation}
where we recall that $\lambda_s=\lambda(1-\frac{s}{S-1})$.

 For $s<S-1$, assume \eqref{eq:hit} is true for $\mathbb{E}\tau_{s+1}$:
\begin{align*}
 \mathbb{E}\tau_{s}&= 1+ \mathbb{E}\tau_{s+1}\frac{\lambda_s}{U}+\mathbb{E}\tau_{s}\frac{1-\mu_s-\lambda_s}{U}+\mathbb{E}\tau_{s-1}\frac{\mu_s}{U}\\
 &= 1+\mathbb{E}\tau_{s}\frac{1-\mu_s-\lambda_s}{U}+\mathbb{E}\tau_{s-1}\frac{\mu_s}{U}+\mathbb{E}\tau_s\frac{\lambda_s}{U}+\sum_{s'= s+1}^{S-1}\prod_{i=s+1}^{s'}\frac{\lambda_{i-1}}{\mu_i}\\
 &=\frac{U}{\mu_s}+\mathbb{E}\tau_{s-1}+\frac{U}{\mu_s}\sum_{s'= s+1}^{S-1}\prod_{i=s+1}^{s'}\frac{\lambda_{i-1}}{\mu_i} \text{ by gathering the $\tau_s$ terms }\\
 &=\mathbb{E}\tau_{s-1}+\frac{U}{\mu_s}\sum_{s'= s}^{S-1}\prod_{i=s+1}^{s'}\frac{\lambda_{i-1}}{\mu_i},
\end{align*}
the induction is therefore true, and by definition of $m^\pi(0)$ we have: $\mathbb{E}\tau_s\leq \mathbb{E}\tau_{s-1}+\frac{U}{\mu_s}m^{\pi}(0)^{-1}$.
\end{proof}

\begin{Lem}
 \label{pro:bias} 
 For any policy $\pi$, define for $s \in \croc{1,\ldots, S-1}$ the variation of the bias 
 \begin{equation*}
 \partial H^\pi(s):=H^{\pi}(s)-H^{\pi}(s-1)=\sum_{t=1}^{\infty} \prt{P^t(s,\cdot)-P^t(s-1,\cdot)}\mathbf{r}.
 \end{equation*}
 $$\partial H^\pi(s)\leq\Delta(s):= 2r_{\max}e^{\frac{\lambda}{\mu}}(1+ \log s).$$
\end{Lem}

\begin{proof}
 We will use an optimal coupling, that is, a coupling such that the following infimum is reached, as defined in \cite{Peres-2008}.
\begin{equation}
 \left\|P^t(s,\cdot)-P^t(s-1,\cdot) \right\|_{\rm TV}=\inf\{\mathbb{P}(X\neq Y): (X_t,Y_t) \text{ is a coupling of $P^t(s,\cdot)$ and $P^t(s-1,\cdot)$}\}.
 \label{eq:coupling1}
\end{equation}
Let $X$ and $Y$ be Markov chains with the same transition matrix $P$ and starting states $X_1=s$, $Y_1=s-1$. We couple $X$ and $Y$ in the following way: 
For $t\geq2$, let $U_t \sim \mathcal{U}([0,1])$ be a sequence of independent random variables sampled uniformly on $[0,1]$. We have:
\begin{equation}
X_{t+1}=
\begin{cases}
 X_t-1 & \text{if } U_t \leq \mu_{X_t}\\
 X_t & \text{if } \mu_{X_t}\leq U_t \leq 1-\lambda_{X_t}\\
 X_t+1 &\text{if } 1-\lambda_{X_t} \leq U_t,
\end{cases}
\end{equation}
and define $Y_{t+1}$ the same way from $Y_t$. This coupling is optimal, but in particular we have from \eqref{eq:coupling1}, with $\delta_r:=\max_{s,a,a'} |r(s,a)-r(s-1,a')|=\frac{C\mu+w(\amax)}{U}$:
:
$$(P^t(s,\cdot)-P^t(s-1,\cdot))\mathbf{r} \leq 2\mathbb{P}(X_t \neq Y_t)\delta_r. $$

 As $\tau_{s}$ is the time needed for $X_t$ to hit $0$ starting from $s$, the coupling time is lower than $\tau_{s}$:
 \begin{equation*}
  \mathbb{P}(X_t \neq Y_t) \leq \Pb\prt{\tau_{s\to 0} >t},
 \end{equation*}
so that summing over $t$ gives:  
\begin{equation*}
 \partial H^\pi(s) \leq 2\delta_r \mathbb{E}\tau_{s},
\end{equation*}
and now using Lemma~\ref{Lem:hit0}:
\begin{equation*}
 \partial H^\pi(s)\leq 2\delta_r m^{\pi}(0)^{-1} \sum_{i=1}^s \frac{U}{\pi(i)+\mu i}.
\end{equation*}
Finally, using $m^{\pi}(0)^{-1}\leq m^{\pi^0}(0)^{-1}\leq e^{\frac{\lambda}{\mu}}$:

\begin{equation*}
 \partial H^\pi(s)\leq 2r_{\max}e^{\frac{\lambda}{\mu}} \sum_{i=1}^s \frac{1}{i},
\end{equation*}
so that:
\begin{equation*}
 \partial H^\pi(s)\leq 2r_{\max}e^{\frac{\lambda}{\mu}} (1+ \log s)=\Delta(s),
\end{equation*}
\end{proof}

\subsection{From bias variations to probability transition variations}\label{ssec:biastoprob}

The three first lemmas of this subsection are used in the proof of Lemma~\ref{Lem:bias difference}. This lemma is needed to obtain equation \eqref{eq:biasbias}.
\begin{Lem}
  For a MDP with rewards $r \in [0,r_{\max}]$ and transition matrix $P$, denote by $J_s(\pi,T):=\E \brac{\sum_{t=0}^T r(s_t,\pi(s_t))}$ the expected cumulative rewards until time $T$ starting from state $s$, under policy $\pi$. Let  $D_\pi$ be the diameter under policy $\pi$. 
 The following inequality holds: $\mbox{span }(J(\pi,T)) \leq r_{\max}D_\pi$.
 \label{Lem:span}
\end{Lem}

\begin{proof}
 Let $s,s' \in \mathcal{S}.$ Call $\tau_{s\to s'}$ the random time needed to reach state $s'$ from state $s$ under policy $\pi$. Then:
 \begin{align*}
  J_s(\pi,T) &= \E\brac{\sum_{t=0}^T r(s_t)}\\
  &= \E\brac{\sum_{t=0}^{\tau_{s\to s'}-1} r(s_t)}+\E\brac{\sum_{t=\tau_{s \to s'}}^T r(s_t)}\\
  &\leq r_{\max}\E\brac{\tau_{s \to s'}}+J_{s'}(\pi,T) \\
  &\leq r_{\max} D_{\pi} +J_{s'}(\pi,T),
 \end{align*}
which proves the lemma.
\end{proof}

\begin{Lem}
 Consider two ergodic MDPs $M$ and $M'$. For  $i \in {1,2},$ let $r_i\in [0,r_{\max}]$ and $P_i$ be the rewards and transition matrix of  MDP $M_i$ under policy $\pi_i$, where both MDPs have the same state and action spaces. Denote by $g_i$ the  average reward obtained under policy $\pi_i$ in the MDP $M_i$. Then  the difference of the gains is upper bounded.
 $$|g-g'| \leq \| r-r'\|_{\infty} + r_{\max}D_{\pi} \| P-P'\|_\infty.$$
 \label{Lem:gain difference}
\end{Lem}
\begin{proof}
Define for any state $s$ the following correction term  $b(s):=r_{\max}D_{\pi}\|p(\cdot|s)-p'(\cdot|s)\|.$ Let us show by induction that for $T \geq 0$,
$$ \sum_{t=0}^{T-1} P^t r \leq \sum_{t=0}^{T-1} P'^t (r+b) .$$
This is true for $T=0$. Assume that the inequality is true for some $T\geq0$, then
\begin{align*}
 \sum_{t=0}^{T} P^t r - \sum_{t=0}^{T} P'^t (r+b) &=-b + P\sum_{t=0}^{T-1} P^t r - P'\sum_{t=0}^{T-1} P'^t (r+b)\\
 &=-b + P'\prt{ \sum_{t=0}^{T-1} P^t r -\sum_{t=0}^{T-1} P'^t (r+b)} + (P-P')\sum_{t=0}^{T} P^t r\\
 &\leq -b + (P-P')\sum_{t=0}^{T} P^t r \text{ by induction hypothesis.} 
\end{align*}
Notice that, for any state $s$:
\begin{align*}
 \prt{(P-P')\sum_{t=0}^{T} P^t r }(s) & \leq \|p(\cdot|s)-p'(\cdot|s)\| \cdot \mbox{span }(J(T))\\
 &\leq r_{\max}D_{\pi}\|p(\cdot|s)-p'(\cdot|s)\| \text{ by Lemma~\ref{Lem:span}}\\
 &= b(s).
\end{align*}
In the same manner we show that:
$$ \sum_{t=0}^{T} P^t r \geq \sum_{t=0}^{T} P'^t (r-b).$$
Hence, as $P'$ has non-negative coefficients, denoting by $e$ the unit vector:
$$\left\|\sum_{t=0}^{T} P^t r - \sum_{t=0}^{T} P'^t r\right\|_\infty \leq \|b\|_\infty \left\|\sum_{t=0}^{T}P'^t \cdot e \right\|_\infty=\|b\|_\infty (T+1).$$
We can also show that:
\begin{equation*}
 \left\|\sum_{t=0}^{T} P'^t r - \sum_{t=0}^{T} P'^t r'\right\|_\infty = \left\|\sum_{t=0}^{T} P'^t (r-r')\right\|_\infty \leq \|r-r'\|_\infty (T+1).
\end{equation*}
And therefore with a multiplication by $\frac{1}{T+1}$ and by taking the Ces\'aro limit in $ \left\|\sum_{t=0}^{T} P^t r - \sum_{t=0}^{T} P'^t r'\right\|_\infty $, and with a triangle inequality:
$$|g-g'| \leq \| r-r'\|_{\infty} + \|b\|_\infty ,$$
where $\|b\|_\infty=r_{\max}D_{\pi} \| P-P'\|_\infty$.
\end{proof}

\begin{Lem}
  Let $P$ be the stochastic matrix of an ergodic Markov chain with state space $1,\dots, S$. The matrix $A:=I-P$ has a block decomposition
  $$A =\begin{pmatrix}
 A_S &  b \\
 c & d\end{pmatrix};
 $$
 then  $A_S$,  of  size $(S-1) \times (S-1)$  is invertible and $\| A_S^{-1}\|_{\infty}= \sup_{i\in\mathcal{S}} \E_i \, \tau_S$, where $\E_i\, \tau_S$ is the expected time to reach state $S$ from  state $i$. 
 \label{Lem:hitting time}
\end{Lem}
Remark that this lemma is true for any state in $\mathcal{S}$.
\begin{proof}
$\prt{\E_i\, \tau_S}_i$ is the unique vector solution to the system:
$$\begin{cases}
 v(S)=0 \\
 \forall i \neq S, \, v(i)=1+\sum_{j \in \mathcal{S}} P(i,j)v(j)
\end{cases}$$
We can rewrite this system of equations as: $\tilde{A}v=e-e_S$, where $\tilde{A}$ is the matrix $$\tilde{A}:=\begin{pmatrix}
 A_S &  b \\
 0 & 1\end{pmatrix},$$
 $e$ the unit vector and $e_S$ the vector with  value $1$ for the last state and $0$ otherwise.
 Then $\widetilde{A}$ and $A_S$ are invertible and we write:
 $$\tilde{A}^{-1}=
 \begin{pmatrix}
 A_S^{-1} &  -A_S^{-1}b \\
 0 & 1
 \end{pmatrix}.$$
 Thus, by computing $\tilde{A}^{-1}(e-e_S)$, for $i\neq S$, $\prt{\E_i\, \tau_S}_i= A_S^{-1} e$. By definition of the infinite norm and using that  $A_S$ is an M-matrix and that its inverse has non-negative components, $\| A_S^{-1}\|_{\infty}= \sup_{i\in\mathcal{S}} \E_i \, \tau_S$.
\end{proof}

In the following lemma, we use the same notations as in  Lemma~\ref{Lem:gain difference} with a common state space $\{ 1,\ldots S\}$.
\begin{Lem}
Let the biases $h$, $h'$ be the biases of the two MDPs  verify their respective Bellman equations with the renormalization choice $h(S)=h'(S)=0$. Let $\sup_{s \in \mathcal{S}} \E_{s}\, \tau^\pi_{s'}$  be the  worst expected hitting time to reach the state $s'$ with policy $\pi$, and call $T_{hit}:=\inf_{s' \in \mathcal{S}} \sup_{s \in \mathcal{S}} \E_{s}\,\tau_{s'}$.
 We have the following control of the difference:
$$
 \|h-h'\|_\infty \leq 2 T_{hit} ( D'r_{\max} \|P-P'\|_\infty  + \|r-r'\|_\infty)
 $$
 \label{Lem:bias difference}
\end{Lem}
Notice that although the biases are unique up to a constant additive term, the renormalization choice does not matter as the unit vector is in the kernel of $(P-P')$.
\begin{proof}
The computations in this proof follow the same idea as in the proof of  \cite[Theorem 4.2]{Ipsen-1994}.
 The biases verify the following Bellman equations $r -g e = (I-P)h$, and also the arbitrary renormalization equations, thanks to the previous remark: $h(S)=0$. Using the same notations as in the proof of Lemma~\ref{Lem:hitting time}, we can write the system of equations $\tilde{A}h=\tilde{r}-\tilde{g}$, with 
 $\tilde{r}$ and $\tilde{g}$ respectively equal to $r$ and $g$ everywhere but on the last state, where their value is replaced by $0$.
 
 We therefore have that $h=\tilde{A}^{-1}(\tilde{r}-\tilde{g})$, and with identical computations, $h'=\tilde{A'}^{-1}(\tilde{r'}-\tilde{g'}).$ By denoting $dX:=X-X'$ for any vector or matrix $X$, we get:
 $$dh=-\tilde{A}^{-1}(d\tilde{r}-d\tilde{g}+d\tilde{A}h').$$
 The previously  defined block  decompositions are:
 $$\tilde{A}^{-1}=
 \begin{pmatrix}
 A_S^{-1} &  -A_S^{-1} b \\
 0 & 1
 \end{pmatrix}%
 \quad\text{ and }\quad 
 d\tilde{A}=
 \begin{pmatrix}
 A_S-A_S' &  b-b' \\
 0 & 0
 \end{pmatrix}.$$
 For $s<S$, $dh(s)=-e_s^T A_S^{-1}(dA_S h'+d\tilde{r}-d\tilde{g})$ and $dh(S)=0$.
 Now by taking the norm and using \ref{Lem:span}:
 $$\|dh\|_\infty\leq \|A_S^{-1}\|_\infty(r_{\max}D'\|dA_S\|_\infty+\|d\tilde{r}\|+|d\tilde{g}|).$$
 Notice that $\|dA_S\|_\infty\leq\|dP\|_\infty$, $\|d\tilde{r}\|\leq\|dr\|$ and $\|d\tilde{g}\|=|dg| $. Using Lemma~\ref{Lem:gain difference} and Lemma~\ref{Lem:hitting time}, and taking the infimum for the choice of the state of renormalization implies the claimed inequality for the biases.
\end{proof}

\subsection{A McDiarmid's inequality}\label{ssec:McDiarmid}
\begin{Lem}
 Recall that $m^{\max}$ is the stationary measure of the Markov chain under policy $\pi^{\max}$, such that for every state $s$: $\pi^{\max}(s)=\amax$.

 Let $k$ be an episode, and assume that the length of this episode $I_k$ is at least $I(T)=1+\max\croc{Q_{\max},T^{1/4}}$, with $Q_{\max} := \left(\frac{10D}{m^{\max}(S-1)}\right)^2\log \left(\left(\frac{10D}{m^{\max}(S-1)}\right)^4\right)$. Then, with probability at least $1-\frac{1}{4T}$:
 \begin{equation*}
  \nu_k(x_k,a_k)\geq m^{\max}(S-1)I_k-5D\sqrt{I_k\log I_k}.
 \end{equation*}
\label{Lem:McDiarmid-like}
\end{Lem}

We will now prove Lemma~\ref{Lem:McDiarmid-like}:

\begin{proof}
Let $k$ be an episode such that $I_k\geq I(T)$, and first consider it is of fixed length $I$. Denote by $\mathring{r}$ the vector of reward equal to $1$ if the current state is $x_k$ and $0$ otherwise. Denote by $\mathring{g}_{\pi_k}$ the gain associated to the policy $\pi_k$ for the transitions $p$ and rewards $\mathring{r}$, and similarly define $\mathring{h}_{\pi_k}$ the bias, translated so that $\mathring{h}_{\pi_k}(S-1)=0$. Notice in that case, that if we denote by $m_k$ the stationary distribution under policy $\pi_k$, that $m^{\max}(S-1)\leq m_k(s)$ for any state $s$, assuming that $S\geq \frac{\lambda}{\mu}+1$.
Then:
 \begin{align*}
\nu_k(x_k,a_k)&=\sum_{u=t_k}^{t_{k+1}-1} \mathring{r}(s_u)\\
&=\sum_{u=t_k}^{t_{k+1}-1} \mathring{g}_{\pi_{k}}+\mathring{h}_{\pi_{k}}(s_u)-\left\langle p\prt{\cdot| s_u,\pi_k(s_u)},\mathring{h}_{\pi_{k}}\right\rangle \text{ using a Bellman's equation}\\
&=\sum_{u=t_k}^{t_{k+1}-1} \mathring{g}_{\pi_{k}}+\mathring{h}_{\pi_{k}}(s_u)-\mathring{h}_{\pi_{k}}(s_{u+1})+\mathring{h}_{\pi_{k}}(s_{u+1})-\left\langle p\prt{\cdot| s_u,\pi_k(s_u)},\mathring{h}_{\pi_{k}}\right\rangle.
 \end{align*}

By Azuma-Hoeffding inequality \ref{Lem:Azuma}, following the same proof as in section 4.3.2 of \cite{jaksch-2010}, notice that $X_u=\mathring{h}_{\pi_{k}}(s_{u+1})-\left\langle p\prt{\cdot| s_u,\pi_k(s_u)},\mathring{h}_{\pi_{k}}\right\rangle$ form a martingale difference sequence with $|X_u|\leq D$:
$$
\Pb\croc{\sum_{u=t_k}^{t_{k+1}-1} X_u \geq D\sqrt{10I\log I}} \leq \frac{1}{I^5}.
$$
Using that $\left|\mathring{h}_{\pi_{k}}(s_{t_k})-\mathring{h}_{\pi_{k}}(s_{t_{k+1}})\right|\leq D$, 
with probability at least $1-\frac{1}{I^2}$:
\begin{equation*}
 \nu_k(x_k,a_k)\geq \sum_{u=t_k}^{t_{k+1}-1} \mathring{g}_{\pi_{k}} - 5D\sqrt{I\log I}.
\end{equation*}
On the other hand:
\begin{equation*}
 \sum_{u=t_k}^{t_{k+1}-1} \mathring{g}_{\pi_{k}} = \nu_k(s_k,a_k) m_k(x_k),
\end{equation*}
so that, using that $m_k(x_k) \geq m^{\max}(S-1)$, with probability at least $1-\frac{1}{I^5}$:
\begin{equation*}
 \nu_k(x_k,a_k)\geq m^{\max}(S-1)I - 5D\sqrt{I\log I}.
\end{equation*}
We now use a union bound over the possible values of the episode lengths $I_k$, between $I(T)+1$ and $T$:
\begin{align*}
 \Pb\croc{\nu_k(x_k,a_k)< m^{\max}(S-1)I_k - 5D\sqrt{I_k\log I_k}}&\leq \sum_{I=I(T)+1}^T \frac{1}{I^5}\leq \sum_{I=T^{1/4}+1}^T \frac{1}{I^5} \\
 &\leq \frac{1}{4T},
\end{align*}
so that we now have that with probability at least $1-\frac{1}{4T}$:
\begin{equation*}
 \nu_k(x_k,a_k)\geq m^{\max}(S-1)I_k - 5D\sqrt{I_k\log I_k}.
\end{equation*} 
\end{proof}

We can show a corollary of Lemma~\ref{Lem:McDiarmid-like} that we will use for the regret computations:
\begin{Cor}
For an episode $k$ such that its length $I_k$ is greater than $I(T)$,with probability at least $1-\frac{1}{4T}$:
 \begin{equation*}
    \nu_k(x_k,a_k)\geq \frac{m^{\max}(S-1)}{2}I_k.
 \end{equation*}
 \label{Cor:worst count}
\end{Cor}

\begin{proof}
 With Lemma~\ref{Lem:McDiarmid-like}, it is enough to show that $5D\sqrt{I_k\log I_k} \leq \frac{m^{\max}(S-1)}{2}I_k$, \emph{i.e.} that $ \sqrt{\frac{I_k}{\log I_k}} \geq\frac{10D}{m^{\max}(S-1)}=:B$. By monotonicity, as $I_k\geq Q_{\max}={B}^2\log {B}^4$ we can show instead that ${B}^2 \log {B}^4 \geq {B}^2 \log \prt{ {B}^2\log {B}^4}$.

 This last inequality is true, using that $\log x \geq \log(2 \log x)$ for $x > 1$. This proves the corollary.
\end{proof}
\end{document}